\documentclass[preprint, 12pt]{elsarticle}
\journal{elsevier}
\usepackage{xr}
\usepackage{array,bm,amsmath,amssymb,mathrsfs,epsfig,epstopdf}
\usepackage{amsthm}
\usepackage{setspace}
\usepackage{CJK}
\usepackage{color}
\usepackage{lscape}
\usepackage{verbatim}
\usepackage{enumerate}
\usepackage{booktabs}
\usepackage{threeparttable}
\usepackage{multicol}
\usepackage{longtable}
\usepackage{multirow}
\usepackage{bm}
\usepackage{bbm}
\usepackage{cases}
\usepackage{caption}
\usepackage{natbib}
\usepackage{algorithm}
\usepackage{algorithmic}
\usepackage{geometry}
\usepackage{framed}
\usepackage[pdftex,bookmarksnumbered,bookmarksopen,colorlinks, citecolor=blue,linkcolor=blue,urlcolor=blue]{hyperref}
\usepackage{graphicx}
\usepackage{amsmath}
\usepackage{amsfonts}
\usepackage{float}
\usepackage{ragged2e}
\usepackage[figuresright]{rotating}
\usepackage{mathtools}
\usepackage{pifont}

\newcommand\Vector[1]{\bm{#1}}

\newcommand\ve{{\Vector{e}}}

\newcommand\MATRIX[1]{\bm{#1}}

\newcommand\mZ{{\MATRIX{Z}}}

\newcommand\mTheta{{\MATRIX{\Theta}}}

\newcommand\gA{{\mathcal{A}}}

\newcommand\gL{{\mathcal{L}}}

\newcommand\vecn{\mathrm{vec}}
\newcommand\tp{\mathrm{pool}}

\DeclareMathOperator*{\argmax}{arg\,max}
\DeclareMathOperator*{\argmin}{arg\,min}

\newtheorem{theorem}{Theorem}[section]
\newtheorem{lemma}{Lemma}[section]

\newcommand{\tabincell}[2]{\begin{tabular}{@{}#1@{}}#2\end{tabular}}

\begin{document}
\title{\bf Adaptive Classification of Interval-Valued Time Series}

\author[1]{Wan Tian}
\ead{wantian61@foxmail.com}

\author[1,2]{Zhongfeng Qin\corref{cor1}}
\ead{qin@buaa.edu.cn}

\cortext[cor1]{Corresponding author}
\address[1]{School of Economics and Management, Beihang University, Beijing 100191, China}
\address[2]{Key Laboratory of Complex System Analysis, Management and Decision (Beihang University), Ministry of Education, Beijing 100191, China}

\begin{abstract}

In recent years, the modeling and analysis of interval-valued time series have garnered significant attention in the fields of econometrics and statistics. However, the existing literature primarily focuses on regression tasks while neglecting classification aspects. In this paper, we propose an adaptive approach for interval-valued time series classification. Specifically, we represent interval-valued time series using convex combinations of upper and lower bounds of intervals and transform these representations into images based on point-valued time series imaging methods. We utilize a fine-grained image classification neural network to classify these images, to achieve the goal of classifying the original interval-valued time series. This proposed method is applicable to both univariate and multivariate interval-valued time series. On the optimization front, we treat the convex combination coefficients as learnable parameters similar to the parameters of the neural network and provide an efficient estimation method based on the alternating direction method of multipliers (ADMM). On the theoretical front, under specific conditions, we establish a margin-based multiclass generalization bound for generic CNNs composed of basic blocks involving convolution, pooling, and fully connected layers. Through simulation studies and real data applications, we validate the effectiveness of the proposed method and compare its performance against a wide range of point-valued time series classification methods.

\end{abstract}

\maketitle	
\section{Introduction} \label{sec1}
Interval-valued time series have attracted significant attention in the fields of statistics and econometrics in recent years \cite{arroyo2006introducing, san2007imlp, maia2008forecasting, han2012autoregressive, gonzalez2013constrained, han2016vector}, as they can simultaneously capture variation and level information. In practical applications, interval-valued time series are quite common. For example, in macroeconomics, the minimum and maximum annualized monthly GDP growth rates form interval-valued data for annual GDP growth rate. In meteorology, interval-valued time series are widely used to describe daily weather conditions, such as pollutant concentrations and temperature. Other common examples include determining the bid and ask prices during a trading period, inflation rates, and short-term and long-term interest rates. In general, interval-valued time series modeling offers two main advantages over point-valued time series \cite{han2016vector}. Firstly, within the same time period, interval-valued time series contain more variation and level information \cite{han2012autoregressive, gonzalez2013constrained, han2016vector}, which means that modeling interval-valued time series can lead to more efficient estimation and powerful inference. Secondly, specific disturbances, which may be considered noise in point-valued time series modeling and have adverse effects, can be addressed through modeling interval-valued time series. 

Over the past three decades, numerous methods for modeling and analyzing univariate and multivariate interval-valued time series, particularly focusing on regression, have been proposed. For example, \citet{maia2008forecasting} represented interval-valued time series using bivariate central and range point-valued time series, and separately applied autoregressive model, autoregressive integrated moving average model, artificial neural network, or combinations of these models to model these point-valued time series, aiming to achieve the purpose of modeling the original interval-valued time series. In the framework of Interval arithmetic, \citet{arroyo2007exponential} extended the classical exponential smoothing methods to interval-valued time series, and compared them with interval multilayer perceptron and other classical methods under the representations of upper-lower bounds and center-range. Similarly, \citet{san2007imlp} proposed an interval multilayer perceptron model based on interval arithmetic, with both input and output of the model being intervals. However, the weights and biases of the model remained point-valued. All of the above methods essentially represented intervals using bivariate point values and then modeled them using point-based methods. Recently, there have been some studies that treated intervals as a whole entity. \citet{han2012autoregressive} introduced the concept of extended random intervals, upon which they developed an autoregressive conditional interval model and a corresponding minimum distance estimation method. Regarding multivariate interval-valued time series, \citet{han2016vector} proposed an interval-valued vector autoregressive moving average model based on the extended random interval concept introduced by \citet{han2012autoregressive}. They also established a minimum distance estimation method for the model parameters and provided a theory of consistency, asymptotic normality, and asymptotic efficiency of the estimator.

In the field of supervised learning, classification and regression are two subtasks with equal importance. However, to the best of our knowledge, very few studies have attempted to address the issue of classification for interval-valued time series.  Certainly, a few classification methods for interval-valued data have been proposed. For example, \citet{palumbo1999non} generalized factorial discriminant analysis to interval-valued scenario and proposed a three-stage discrimination procedure. \citet{rasson2000symbolic} introduced the Bayesian discriminant rule for interval-valued data based on class prior probabilities and kernel density estimation. \citet{qi2020interval} proposed a unified representation framework for interval-valued data and applied various machine learning models for classification. Clearly, the above methods were not specifically designed for interval-valued time series, and did not take into account the sequential dependency characteristics of the data.

For univariate and multivariate point-valued time series, the classification
spanning across statistical learning methods and deep learning-based approaches has been
extensively studied \cite{ismail2019deep, abanda2019review, ruiz2021great}. These methods have demonstrated remarkable performance on public datasets and real-world applications \cite{bagnall2017great}, and can be roughly categorized into five types: distance-based methods \cite{bagnall2017great, lines2015time, cuturi2017soft, lucas2019proximity}, dictionary-based methods \cite{schafer2015boss, schafer2016scalable, schafer2017fast}, Shapelets-based methods \cite{schafer2017fast}, ensemble learning-based methods \cite{bagnall2017great, lucas2019proximity, Bagnall2015TimeSeriesCW}, and deep learning-based methods \cite{ismail2019deep, Wang2016TimeSC, Cui2016MultiScaleCN, IsmailFawaz2019AdversarialAO}. The distance-based methods primarily rely on specific distance metrics to measure the similarity between time series. Among these methods, the combination of Dynamic Time Warping and nearest neighbor classifiers has been dominant in the past \cite{bagnall2017great}. Currently, the state-ofthe-art algorithm is the Elastic Ensemble  \cite{lines2015time}, but its high training complexity restricts its application on large-scale data. The dictionary-based methods classify time series based on the frequency of repetition of certain subsequences. Among this class of methods, the Bag-of-SFA-Symbols method \cite{schafer2015boss} stands out prominently. The Shapelets-based methods identify specific classes by using relatively short and repetitive subsequences. The ensemble learning methods construct different classifiers on various time series representations and combine them together. The method known as Hierarchical Vote Collective of Transformation-Based Ensembles \cite{Lines2016HIVECOTETH} has achieved state-of-the-art performance on the UCR archive \cite{Dau2018TheUT}. Deep learning has attracted considerable attention in
time series mining recently due to its excellent representation capability, and a commonly used method is InceptionTime \cite{IsmailFawaz2019InceptionTimeFA}, a time series classifier that ensemble five deep learning models. More classification methods for point-valued time series can be found in \cite{ismail2019deep, abanda2019review, ruiz2021great}.

However, the above effective methods have not been effectively generalized and applied to interval-valued time series. In this paper, we propose an adaptive classification method for interval-valued time series, which combines imaging methods and deep learning techniques. We represent interval-valued time series by taking the convex combination of upper bound point-valued time series and lower bound point-valued time series, and transform these representations into images using the point-valued time series imaging method Recurrence Plot (RP) \cite{eckmann1995recurrence} for classification, aiming to achieve the classification of the original interval-valued time series. Considering the high similarity within and between classes in the obtained image dataset, we select a fine-grained image classification network as the classifier. Taking into account the constraints imposed on the convex combination coefficients, we formalize the parameter estimation problem as a constrained optimization problem. During the optimization process, we treat the convex combination coefficients as learnable parameters similar to network parameters and provide an efficient parameter estimation method based on alternating direction method of multipliers (ADMM) \citep{Boyd2011DistributedOA}. The proposed method is also applicable to multivariate interval-valued time series by simply replacing RP with the multivariate time series imaging method called Joint Recurrence Plot (JRP) \cite{thiel2004much}. Theoretically, under specific conditions, we establish a margin-based multiclass generalization bound for generic CNNs composed of basic blocks involving convolution, pooling, and fully connected layers. The efficacy of the proposed methodologies is showcased through simulation studies encompassing various data generation processes and real-world data applications.

The structure of this paper is outlined as follows. Section \ref{sec2method} introduces the proposed adaptive classification methods for both univariate and multivariate interval-valued time series, along with the parameter estimation method based on ADMM in Section \ref{sec3ADMM}. In Section \ref{sec4Theoretical}, we present a margin-based multiclass generalization bound for generic CNNs. Simulation studies and real-world data applications are provided in Sections \ref{sec5simulate} and \ref{sec6rreal}, respectively. Concluding remarks are presented in Section \ref{sec7}, and all theoretical proofs are documented in \ref{appendixA}.

\section{Methodologies} \label{sec2method}
In this section, we introduce the proposed adaptive classification methods for univariate and multivariate interval-valued time series. As the subject of classification, a stochastic interval-valued time series \(\{X_t\}^T_{t=1}\) is a sequence of interval-valued random variables indexed by time \(t\). The interval-valued time series can be represented using bivariate point-valued time series, such as upper and lower bounds, i.e., \(X_t = [X^l_t, X^u_t]\).

\subsection{Classification of univariate interval-valued time series} \label{subsec21}
Given a labeled dataset of univariate interval-valued time series \(S_1 = \{(X_i,Y_i)\}^n_{i=1}\), where \(X_i = (X_{i,t})^T_{t=1} = ([X^l_{i,t}, X^u_{i,t}])^T_{t=1}\) represents the univariate interval-valued time series of length \(T\), and \(Y_i\) is its corresponding label. Interval-valued time series can also be represented using center and range, i.e., \(X_{i, t} = (X^c_{i,t}, X^r_{i,t})\), with \(X^c_{i,t} = (X^l_{i,t} + X^u_{i,t}) / 2\) and \(X^r_{i,t} = (X^u_{i,t} - X^l_{i,t}) / 2\). Our objective is to construct a classification model with excellent generalization performance. 

To utilize the point-valued time series imaging and classification methods, we need to represent intervals with representative points. In general, interval can be represented using various representative points, such as the center, range, upper bound, lower bound, or other internal points. However, using the same point representation for interval at different times in interval-valued time series may not be the most discriminative choice. For example, if we represent intervals using the upper bound and then classify the dataset based on the upper bound point-valued time series, it may not achieve optimal classification performance. This has inspired our research on adaptive classification of interval-valued time series, where different points are used to represent intervals at different times. Specifically, we represent intervals using convex combinations of upper and lower bounds. This approach also offers interpretability advantages as convex combinations always lie within the interval.

The inputs in dataset \(S_1\) can be represented by the following two point-valued matrices,
\[
X^l = (X^l_{i, t})_{1 \leq i \leq n, 1 \leq t\leq T }, \ X^u = (X^u_{i, t})_{1 \leq i \leq n, 1 \leq t\leq T },
\]
where \(X^l\) and \(X_u\) are matrices composed of the upper and lower bounds of intervals, respectively.
For instance, the first row of \(X^l\) and \(X^u\) consists of the upper and lower bounds of the interval-valued time series \(X_1\), and so forth. Then, the convex combination matrix \(C = (C_{i,t})_{1 \leq i \leq n, 1 \leq t\leq T }\) of all inputs can be represented as:
\[
C = (C_1,\cdots, C_n)^\top =  X^l \text{diag}(\alpha_1, \cdots, \alpha_T) + X^u \text{diag}(1-\alpha_1, \cdots, 1-\alpha_T) \in \mathbb{R}^{n\times T},
\] 
where \(\alpha = (\alpha_1, \alpha_2, \cdots, \alpha_T) \in [0,1]^T\) are convex combination coefficients that need to be estimated, the \((i, t)\)-th element in matrix \(C\) can be represented as \(C_{i,t} = \alpha_{t}X^l_{i,t} + (1-\alpha_t)X^u_{i,t}\), and \(\text{diag}(\alpha_1, \alpha_2, \cdots, \alpha_T)\) is a diagonal matrix with \(\alpha_1, \alpha_2, \cdots, \alpha_T\) as its diagonal elements. When all convex combination coefficients are either 0 or 1, the convex combination matrix is composed of the upper bound time series or the lower bound time series respectively. When the convex combination coefficients are all 1/2, the convex combination matrix is composed of the central time series. 

In this case, we can directly apply point-valued time series classification methods to dataset \(\{(C_i,Y_i)\}^n_{i=1}\) to achieve the classification goal of the original interval-valued time series dataset. In order to effectively leverage computer vision techniques from deep learning to enhance the classification performance of interval-valued time series, we further employ the point-valued time series imaging method RP to transform convex combination time series into an image. We should note that there are multiple methods available to transform time series into images, such as Gramian Angular Summation/Difference Field, and Markov Transition Field \cite{wang2015imaging}, each capturing different features of the time series from different perspectives. Recurrence is a fundamental characteristic of dynamic systems, and it can effectively describe the behavior of a system in phase space \cite{eckmann1995recurrence}. On the other hand, RP is a powerful tool for visualizing and analyzing such behavior and has a good informationtheoretic interpretation \cite{eckmann1995recurrence}. This is the reason why we chose RP as the imaging method.

For convex combination time series \(C_i\), RP first extracts trajectories of length \(m\) and time delay \(\kappa\) as follows, 
\[
\vec{C}_{i,j} = \left(C_{i,j}, C_{i,j+\kappa}, \cdots, C_{i,j + (m-1)\kappa}\right), \ j = 1,2,\cdots, T - (m-1)\kappa,
\]
then, RP defines the distance between pairs of trajectories of \(C_i\) as
\begin{equation} \label{imagingpro}
R_i(j, k) = H(\epsilon_i - \lVert \vec{C}_{i,j} -\vec{C}_{i,k}\rVert_2),\ j,k = 1,2,\cdots, T - (m-1)\kappa, 
\end{equation}
where \(H(\cdot)\) represents the Heaviside function, \(\lVert \cdot \rVert_2\) denotes Euclidean norm, and \(\epsilon_i\) is the threshold value that needs to be specified, which may vary for each convex combination time series. From the imaging process of RP, it can be seen that the threshold value is a crucial hyperparameter determining the amount of information extracted. For instance, if the threshold value is set to a value greater or smaller than the distance between any two trajectories, the resulting RP will be a matrix with all ones or all zeros. This indicates that very little or no pattern information is extracted from the time series. Various strategies can be employed to select the threshold, among which the commonly used approach is based on the quantiles of distances. Using the imaging method RP, we can transform the original interval-valued time series dataset \(S_1\) into an image dataset \(I_1 = \{(R _i,Y_i)\}^n_{i=1}\) with \(R_i  = (R_i(j, k))_{1 \leq  j, k\leq T-(m-1)\kappa}\). We proceed to use a convolutional neural network (CNN) to classify \(I_1\) to achieve the goal of classifying the original interval-valued time series. We discuss the selection of specific network architectures in Section \ref{sec5simulate}.

\subsection{Classification of multivariate interval-valued time series}
In this section, we extend the multivariate point-valued time series imaging method JRP to multivariate interval-valued time series, and perform discriminative analysis based on this extension. Similar to Section \ref{subsec21}, we have a given labeled multivariate interval-valued time series dataset denoted by \(S_2 = \{(Z_i, Y_i)\}^{n}_{i = 1}\), which is used for classification. In the dataset \(S_2\), {for each \(i\) with \(1 \leq i \leq n\), the input} \(Z_i = (Z_{i, j, t})_{1\leq j \leq p, 1\leq t \leq T}\) is an multivariate interval-valued time series with a length of \(T\) and a dimension of \(p\). Using simple interval arithmetic, the \(i\)-th input \(Z_i\) can be represented as
\[
Z^l_i = (Z^l_{i, j, t})_{1\leq j \leq p, 1\leq t \leq T},\ Z^u_i = (Z^u_{i, j, t})_{1\leq j \leq p, 1\leq t \leq T}, 
\]
where \(Z^l_i\) and \(Z^u_i\) are respectively composed of the upper and lower bounds of the multivariate interval-valued time series \(Z_i\). 

Similar to handling univariate interval-valued time series, we represent the interval using convex combinations of its upper and lower bounds,
\[
D_i = (D_{i, j, t})_{1\leq j \leq p, 1\leq t \leq T} = \text{diag}(\alpha_1, \alpha_2, \cdots, \alpha_p) Z^l_i + \text{diag}(1-\alpha_1, 1-\alpha_2, \cdots, 1-\alpha_p) Z^u_i, 
\]
where \(\alpha = (\alpha_1, \alpha_2, \cdots, \alpha_p) \in [0,1]^p\) are convex combination coefficients. Many discriminative methods for multivariate point-valued time series can be directly applied to the dataset \(\{(D_i,Y_i)\}^n_{i = 1}\) to achieve the classification of \(S_2\). However, in our approach, we additionally employ JRP for imaging \(\{(D_i,Y_i)\}^n_{i = 1}\), and then use models such as CNNs for classifying.

JRP is an extension of RP in the context of multivariate point-valued time series. {It} first uses RP to imaging the point-valued time series for each individual dimension, and then utilizes Hadamard product to consolidates information across different dimensions. To be more specific, for the \(j\)-th dimension of the \(i\)-th convex combination matrix \(D_i\), the RP initially extracts trajectories of length \(m\) with a time delay \(\kappa\),
\[
\vec{D}_{i,j,k} = (D_{i,j, k}, D_{i,j, k + \kappa}, \cdots,D_{i,j, k + (m-1)\kappa}),\ k = 1,2, \cdots, T-(m-1)\kappa,
\]

Then, RP calculates the distances between pairs of trajectories on the {\(j\)-th dimension},
\[
R_{i, j}(k, h) = H(\epsilon_{i,j}- \lVert \vec{D}_{i,j,k} - \vec{D}_{i,j,h} \rVert_2),\ k, h = 1,2,\cdots, T - (m-1)\kappa,
\]
where \(\epsilon_{i,j}\) refers to the threshold value of the \(j\)-th dimension for the \(i\)-th convex combination matrix. Finally, JRP consolidates information from different dimensions by
\[
R_i = R_{i, 1} \odot R_{i, 2} \odot \cdots \odot R_{i, p},
\]
and this is the image corresponding to the \(i\)-th convex combination matrix. Directly, we can classify the imaging dataset \(I_2 = \{(R_i,Y_i)\}^n_{i = 1}\) to achieve the goal of classifying the original multivariate interval-valued time series dataset. 

By utilizing the convex combination of upper and lower bounds and the point-value time series imaging method, we can transform the classification problem of interval-valued time series dataset (denoted as \(S_1\) or \(S_2\)) into a classification problem of image dataset \(I = \{(R_i,Y_i)\}^n_{i = 1}\) (denoted as \(I_1\) or \(I_2\)), thereby enhancing classification performance. However, it is important to note that the convex combination coefficients are unknown, and we need effective optimization methods to estimate them.

\section{Alternating direction method of multipliers} \label{sec3ADMM}
In this section, we discuss how to effectively estimate the parameters of the neural network and the convex combination coefficients. We choose a CNN-type neural network, denoted as \(f(\cdot; \Theta)\), to classify the image dataset \(I\), where \(\Theta\) represents the parameters of the network. The specific form of the loss function is discussed in the theoretical part of Section \ref{sec4Theoretical}. The quality of classification is measured by an appropriate loss function \(\ell(f(R; \Theta), Y)\), and the smaller the loss, the better the classification performance. Empirical risk minimization is the most commonly used criterion for estimating model parameters. Taking into account the constraints imposed on convex combination coefficients, our ultimate optimization objective is
\begin{equation} \label{unopt}
\min_{\alpha, \Theta} \gL(\alpha, \Theta)  =  \frac{1}{n}\sum_{i=1}^{n}\ell(f(R_i; \Theta),Y_i), \text{ subject to } \alpha \in \gA,
\end{equation}
where \(\gA\) represents the constraint set for convex combination coefficients, applicable to both univariate and multivariate interval-valued time series, denoted as \([0, 1]^T\) and \([0, 1]^p\), respectively.

For this constrained optimization problem (\ref{unopt}), we can effectively solve it using the ADMM algorithm. Specifically, we first introduce an auxiliary variable \(\beta\) with the same dimension as \(\alpha\), then rewrite the optimization problem (\ref{unopt}) as follows:
\begin{equation} \label{unoptre2}
	\min \mathcal{L}(\alpha, \Theta) + g(\beta),\ \text{ subject to } \alpha - \beta = 0,
\end{equation}
where \(g(\beta)\) is the indicator function of \(\gA\), that is
\[
g(\beta) = \begin{cases}
	0,\ \ \ \ \  \beta \in \gA, \\
	+\infty,\  \text{others}.
\end{cases}
\]

The augmented Lagrangian function for (\ref{unoptre2}) with multiplier \(\lambda\) is
\begin{equation} \label{arglaga}
L_\rho(\alpha, \beta, \Theta, \lambda) = \mathcal{L}(\alpha, \Theta) + g(\beta) + \lambda^\top(\alpha - \beta) + (\rho/2)\lVert \alpha - \beta \rVert^2_2.
\end{equation}

Given four sets of parameters \(\Theta^s, \alpha^s, \beta^s\), and \(\lambda^s\) obtained from the \(s\)-th iteration, the \((s+1)\)-th iteration of the ADMM algorithm consists of the following four steps:
\begin{equation} \label{oriadmm}
\begin{cases}
	\Theta^{s+1} = \argmin\limits_{\Theta} L_\rho(\alpha^s, \beta^s, \Theta, \lambda^s),\\
	\alpha^{s+1} = \argmin\limits_{\alpha} L_\rho(\alpha, \beta^{s}, \Theta^{s+1}, \lambda^s),\\
	\beta^{s+1} = \argmin\limits_{\beta} L_\rho(\alpha^{s+1}, \beta, \Theta^{s+1}, \lambda^s),\\
	\lambda^{s+1} = \lambda^{s} + \tau \rho (\alpha^{s+1} - \beta^{s+1}), 
\end{cases}	
\end{equation} 
where \(\tau\) is the step size and typically belongs to the interval \((0, (1 + \sqrt{5})/2]\). Defining the primal residual \(r = \alpha - \beta\), we express the last two terms on the right-hand side of (\ref{arglaga}) as
\[
\begin{aligned}
\lambda^\top r + (\rho/2)\lVert r \rVert^2_2 &= (\rho/2)\lVert r + (1/\rho)\lambda\rVert^2_2 - (\rho/2)\lVert (1/\rho)\lambda\rVert^2_2\\
&=(\rho/2)\lVert r + u\rVert^2_2 - (\rho/2)\lVert u\rVert^2_2,\\
\end{aligned}
\]
where \(u = (1/\rho)\lambda\) is the scaled dual variable. 

By using the scaled dual variable and ignoring the terms unrelated to the iteration, we can express the iteration procedure (\ref{oriadmm}) as follows:
\begin{align}
\Theta \text{ step: }\Theta^{s+1} &= \argmin\limits_{\Theta} \mathcal{L}(\alpha^s, \Theta), \label{iterd1} \\
\alpha \text{ step: } \alpha^{s+1} &= \argmin\limits_{\alpha} \mathcal{L}(\alpha, \Theta^{s+1})	+ (\rho/2)\lVert \alpha - \beta^s + u^s\rVert^2_2, \label{iterd2}\\
\beta \text{ step: }\beta^{s+1}  &=\Pi_\mathcal{C}(\alpha^{s+1} + u^s), \label{iterd3}\\
u \text{ step: } u^{s+1} &= u^{s} + \tau(\alpha^{s+1} - \beta^{s+1}), \label{iterd4}
\end{align}
where \(\Pi_\mathcal{A}(\alpha^{s+1} + u^s)\) represents the projection of \(\alpha^{s+1} + u^s\) onto \(\mathcal{A}\). We repeat the above cycle until convergence. The \(\beta\) step and \(u\) step in the above iterative process can be easily executed. In the following, we discuss the computation of the \(\Theta\) step and the \(\alpha\) step in detail.

For the \(\Theta\) step, given \(\alpha = \alpha^s\), the convex combination of upper and lower bounds is determined, and consequently, the imaging dataset is also determined. Therefore, we can directly use common optimization algorithms such as stochastic gradient descent (SGD) or its variants \cite{sun2019survey} to update the neural network parameters until convergence.

For the \(\alpha\) step, given \(\Theta = \Theta^{s+1}\), meaning that the parameters of the neural network are fixed, we still use SGD to update the convex combination coefficients. According to (\ref{iterd2}), we primarily need to obtain the gradient of \(\mathcal{L}(\alpha, \Theta^{s+1})\) with respect to \(\alpha\). We know that the convex combination coefficients \(\alpha\) are mainly used during the imaging process; however, the Heaviside function \(H(\cdot)\) is non-differentiable. To address this non-differentiability issue, we use the following function to smoothly approximate \(H(\cdot)\):
\begin{equation} \label{smoothapproximation}
H(x) = \lim_{\nu \to \infty} (1 + \tanh(\nu x))/2,
\end{equation}
where \(\tanh(\cdot)\) represents the hyperbolic tangent function. Figure \ref{app} is the graph of the function \((1 + \tanh(\nu x))/2\) for different values of \(\nu\). It is evident that larger values of \(\nu\) generally provide a better approximation of \(H(\cdot)\). In subsequent empirical analysis, we compared the impact of different approximation levels on classification performance.
\begin{figure}[H]
\centering 
\includegraphics[scale=1]{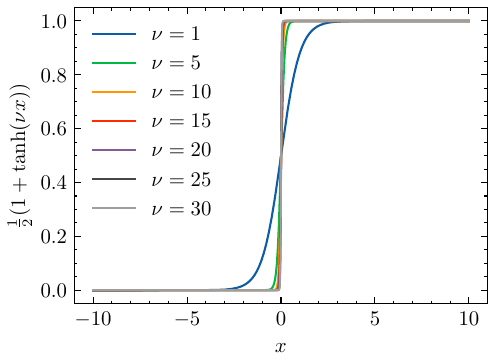}
\caption{The graph of \((1 + \tanh(\nu x))/2\) for various values of \(\nu\).} \label{app}
\end{figure}

Below, we discuss the computation of the \(\Theta\) step using a univariate interval-valued time series as an example. According to the chain rule of differentiation, we can obtain the partial derivatives of the loss function \(\mathcal{L}(\alpha, \Theta^{s+1})\) with respect to each convex combination coefficient as follows:
\begin{equation}\label{derall}
\frac{\partial \mathcal{L}(\alpha, \Theta^{s+1})}{\partial \alpha_t} = \frac{1}{n}\sum_{i=1}^{n} \sum_{j, k} \frac{\partial R_i(j, k)}{\partial \alpha_t} \frac{\partial \mathcal{L}(\alpha, \Theta^{s+1})}{\partial R_{i}(j, k)},\ t = 1,2,\cdots, T. 
\end{equation}

Then, we need to calculate the specific expressions \(\partial R_{i}(j, k) / \partial \alpha_t\) and \(\partial \mathcal{L}(\alpha, \Theta^{s+1}) / \partial R_{i}(j, k)\). For the first partial derivative, it follows from the definition of \(R_{i}(j, k)\) and the smooth approximation of the Heaviside function, we can deduce the following {that}
\begin{equation} \label{partida}
\begin{aligned}
\frac{\partial R_{i}(j, k)}{\partial \alpha_t}  &= \lim_{\nu \to \infty} \frac{1}{2}\frac{\partial (1 + \tanh(\nu(\epsilon_i - \lVert \vec{C}_{i,j} -\vec{C}_{i, k} \rVert_2 )))}{\partial \alpha_t}\\
&= \lim_{\nu \to \infty} -\frac{\nu}{1+\cosh(\nu(\epsilon_i - \lVert \vec{C}_{i,j} -\vec{C}_{i, k} \rVert_2))} \frac{\partial \lVert \vec{C}_{i,j} -\vec{C}_{i, k} \rVert_2}{\partial \alpha_t},
\end{aligned}
\end{equation}
where \(\cosh(\cdot)\) is the hyperbolic cosine function. Depending on the values of \(j\), \(k\), and \(t\), the partial derivatives \(\partial \lVert \vec{C}_{i,j} -\vec{C}_{i, k} \rVert_2 / \partial \alpha_t\) have different expressions, which are provided in \ref{appendixB}. The computation of \(\partial \mathcal{L}(\alpha, \Theta^{s+1}) / \partial R_i(j, k) \) requires consideration of the structure of the CNNs. In general, we can utilize automatic differentiation libraries such as PyTorch (\url{https://pytorch.org/}) for rapid computation. To simplify the analysis, we assume that the CNNs used consist solely of convolutional layers, each comprising convolution, non-linear activation, and pooling operations sequentially, followed by a fully connected layer. The detailed computation process of the first convolutional layer is as follows: 
\[
\begin{aligned}
X_{(1)}(:, :, j) &= Z_{(0)}  \circledast \Theta^j_{(1)} \in \mathbb{R}^{M_1\times N_1	}, \ j= 1,2,\cdots, d_1,\\
Y_{(1)} &=  \sigma_1(X_{(1)}) \in \mathbb{R}^{M_1 \times N_1 \times d_1},\\
Z_{(1)} &= \text{Pool}_{(1)}(Y_{(1)}) \in \mathbb{R}^{\widetilde{M}_{1} \times \widetilde{N}_{1} \times  d_1},
\end{aligned}
\] 
where \(Z_{(0)} = R\) represents the input to the first layer. The notation {\(X_{(1)}(:, :, j)\)} denotes the \(j\)-th feature map of the first layer. The parameter \(\Theta^j_{(1)} \in \mathbb{R}^{k_1 \times k_1}\) represents the \(j\)-th convolutional kernel of size \(k_1 \times k_1\), with a total of \(d_1\) kernels in the first layer. The operations \(\circledast, \text{Pool}_{(1)}(\cdot)\), and \(\sigma_{1}(\cdot)\) denote the convolutional operation, pooling operation, and non-linear activation function (e.g., the Rectified Linear Unit (ReLU)) in the first layer, respectively. Then we have
\[
\begin{aligned} 
\frac{\partial \mathcal{L}(\alpha, \Theta^{s+1})}{\partial R_i} & = \sum_{j=1}^{d_1}\text{rot180}(\Theta^{j}_{(1)}) \widetilde{\circledast} \frac{\partial \mathcal{L}(\alpha, \Theta^{s+1})}{\partial X_{(1)}(:, :, j)}\\
&= \sum_{j=1}^{d_1}\text{rot180}(\Theta^j_{(1)}) \widetilde{\circledast} \delta^{(1, j)},
\end{aligned}
\]
where \(\text{rot180}(\cdot)\) represents a 180-degree rotation, \(\delta^{(1, j)}\) denotes the error term of the \(j\)-th feature map in the first layer, which can be backpropagated to update the parameters \(\Theta\), \(\widetilde{\circledast}\) refers to wide convolution, which means that the convolution operation is performed after padding zeros on both sides of the feature map.

With the above discussion, we can compute the derivatives of the loss function \(\mathcal{L}(\alpha,\Theta^{s+1})\) with respect to the convex combination coefficients, which can then be used to update these coefficients.

\section{Theoretical Results}\label{sec4Theoretical}
In this section, we present the margin-based multiclass generalization bound for the proposed method, which is of utmost interest in statistical learning. Before that, we briefly review some developments in the related field and introduce some basic notations.

Generalization bounds can quantify how well the learning process can generalize to unseen data. Current research primarily constructs the generalization bounds of neural network models from the perspectives of nonparametric statistics \citep{he2020recent, hu2021model, bauer2019deep, bartlett2021deep,neyshabur2017pac, fan2022noise} and functional analysis \citep{pinkus1999approximation, zhou2020theory, fang2020theory, guhring2020error, mao2021theory}. Our paper focuses on constructing generalization bounds from a statistical perspective. Additionally, existing research can also be categorized based on whether the generalization bounds are related to the network structure (e.g., the number of layers, the width of each layer, activation functions, etc.) \citep{he2020recent, duan2023fast}. For instance, \citet{duan2023fast} constructs generalization bounds for parametric and non-parametric regression from the perspective of deep representation, ignoring network structure. The advantage of constructing generalization bounds related to network structure is that it allows theoretical analysis of which components of the network structure may impact the final classification performance. This is a significant advantage for neural networks, which lack interpretability. This is also the focus of our paper. 

A classic approach to constructing generalization bounds begins by establishing bounds on Rademacher complexity. Subsequently, the complexity of the hypothesis space, such as its covering number or VC dimension, is used to bound the Rademacher complexity \citep{geer2000empirical, van2000asymptotic, gine2021mathematical, bartlett2005local}. Common types of Rademacher complexity include global Rademacher complexity \citep{zhang2023mathematical} and offset Rademacher complexity \citep{liang2015learning}, where the latter is a penalized version of the former. Next, a crucial question is how to quantify the complexity of the hypothesis space (in our paper, specific structured CNNs), such as bounds on the covering number.

Without loss of generality, assume \((R_1,Y_1), (R_2,Y_2),\cdots, (R_n,Y_n)\) are independent and identically distributed copies of \((R, Y)\) under an unknown joint measure, where \(R \in \mathbb{R}^{M_0\times N_0}\) is the input image and \(Y \in \{1,2,\cdots, K\}\) is its corresponding label. Similar to the parameter estimation based on the ADMM algorithm in Section \ref{sec3ADMM}, we assume that the network \(f(R; \Theta): \mathbb{R}^{M_0\times N_0} \to\mathbb{R}^K\) used for classification consists of \(L\) convolutional layers and one fully connected layer. Each convolutional block comprises convolution, non-linear activation, and pooling operations. Here, \(\Theta = \{\Theta_{(1)}, \Theta_{(2)}, \cdots, \Theta_{(L)}, W_{L+1}\}\), where \(\Theta_{(l)}\) represents the parameters of the \(l\)-th convolutional layer, and \(W_{L+1}\) represents the parameters of the fully connected layer, i.e., the \((L+1)\)-th layer. Of course, we can consider other network structures, but the corresponding bounds on the covering number would be different. 

The input and output of the \(l\)-th layer are \(Z_{(l-1)} \in \mathbb{R}^{\widetilde{M}_{l-1} \times \widetilde{N}_{l-1} \times d_{l-1}}\) and \(Z_{(l)} \in \mathbb{R}^{\widetilde{M}_{l} \times \widetilde{N}_{l} \times d_{l}}\), respectively, the computation process can be described as follows,
\begin{equation} \label{cnnlyer}
\begin{aligned}
X_{(l)}(:, :, j) &= \Theta^j_{(l)} \circledast Z_{(l-1)} \in \mathbb{R}^{M_l\times N_l}, \ j= 1,2,\cdots, d_l,\\
Y_{(l)} &=  \sigma_l(X_{(l)}) \in \mathbb{R}^{M_l \times N_l \times d_l},\\
Z_{(l)} &= \text{Pool}(Y_{(l)}) \in \mathbb{R}^{\widetilde{M}_{l} \times \widetilde{N}_{l} \times  d_l},
\end{aligned}
\end{equation}
where \(X_{(l)}(:, :, j)\) represents the \(j\)-th feature map of the \(l\)-th layer, \(\Theta^j_{(l)} \in \mathbb{R}^{k_l  \times k_l	\times d_{l-1}}\) refers to the \(j\)-th convolutional kernel with a size of \(k_l \times k_l\), and the \(l\)-th layer consists of a total of \(d_l\) kernels, \(\circledast, \text{Pool}_{(l)}(\cdot)\), and \(\sigma_{l}(\cdot)\) denote the convolutional, pooling, and non-linear activation operations in the \(l\)-th layer, respectively. In particular, \(Z_{(0)} = R\) represents the input image. The last layer is a fully connected layer and is formulated by
\[
f(R; \Theta) =  W_{(L+1)} \text{vec}(Z_{(L)}) \in \mathbb{R}^{K},
\] 
where \(\text{vec}(Z_{(L)})\) represents the vectorization of the output of the last convolutional layer, and \(W_{(L+1)} \in \mathbb{R}^{C \times \widetilde{M}_{L} \widetilde{N}_{L} d_{L}}\) is the weight matrix. 

The network output \(f(R; \Theta)\) is converted into a class label by taking the argmax over its components, which is denoted as \(R \mapsto \argmax_j f(R; \Theta)_j\). For the analysis of margin-based generalization error bounds in multiclass classification, we need to revisit the margin operator, which is defined {by}
\[
\mathcal{M}: \mathbb{R}^K \times \{1,2,\cdots, K\}, \ (V, Y) \mapsto V_Y - \max_{Y^\prime\neq Y} V_{Y^\prime},
\]
and the ramp loss \(\ell(x): \mathbb{R}\mapsto \mathbb{R}^+\) defined as
\[
\ell(x)= \begin{cases}
0, x < -\gamma\\
1+ x/\gamma, -\gamma \leq x \leq 0\\
1, x>0,
\end{cases}
\]
where \(\gamma\) is a predetermined constant. From the definition of the margin operator, if we use the softmax criterion for final classification, it describes the size of the classification gap. It is evident that a larger value of the margin operator indicates better classification performance of the model. Using the ramp loss defined above, the expected risk and its empirical counterpart for the classification model are as follows:
\[
\mathcal{R}(f) = \mathbb{E}(\ell(-\mathcal{M}(f(R; \Theta), Y))), \  \widehat{\mathcal{R}}(f) = \frac{1}{n}\sum_{i=1}^{n}\ell(-\mathcal{M}( f(R_i; \mTheta),Y_i)).
\]

It is evident that the empirical risk \(\widehat{\mathcal{R}}(f)\) is equivalent to the optimization objective in (\ref{unopt}). It should be noted that the image dataset is related to the convex combination coefficients. Let \(\alpha^*\) and \(\Theta^*\) be the convex combination coefficients and network parameters estimated based on the optimization objective (\ref{unopt}), respectively, i.e.,
\[
\alpha^*, \Theta^* \in \argmin_{\alpha, \Theta} \widehat{\mathcal{R}}(f),\ \text{subject to  } \alpha \in \mathcal{A}.
\]

Let \(f^*\) be the abbreviation for \(f(R; \Theta^*)\). In the following, we discuss how to construct the generalization bound of the estimated model while considering the network structure, i.e., the upper bound of
\begin{equation} \label{targenbound}
\mathbb{P}\left(\argmax_{j}f(R; \Theta^*)_j \neq Y\right) -  \widehat{\mathcal{R}}(f^*).
\end{equation}

It should be noted that in (\ref{targenbound}), the input image \(R\) are related to the convex combination coefficients \(\alpha^*\). Therefore, when constructing the generalization bound, we need to consider both \(\alpha^*\) and \(\Theta^*\). Considering these two sets of parameters significantly increases the complexity of constructing the generalization bound. To address this issue, we first establish a generalization bound that holds for any \(\alpha \in \gA\) and \(\Theta\). This naturally applies to the combination of \(\alpha^*\) and \(\Theta^*\). Without loss of generality, we assume that the convex combination coefficients are fixed, which also means that the input images are determined. Therefore, we discuss below how to construct the upper bound of the following expression,
\begin{equation} \label{targenboundreal}
\mathbb{P}\left(\argmax_{j}f(R; \Theta)_j \neq Y\right) -  \widehat{\mathcal{R}}(f).
\end{equation}

Similar to the classic methods discussed above, we first bound expression (\ref{targenboundreal}) using Rademacher complexity, and then we use the covering number of the hypothesis space to bound the Rademacher complexity. Prior to this, we formalize the hypothesis space based on the assumptions made about the network structure. Specifically, for the first \(l\) convolutional layers, the corresponding hypothesis space is
\[
\mathbb{F}_l = \left\{\tp_{(l)}(\sigma_{l}(\Theta_{(l)}\circledast \cdots \tp_{(1)}(\sigma_{1}(\Theta_{(1)} \circledast Z_{(0)}) \cdots))):\lVert W_{(j)} \rVert_\sigma \leq b_j, j = 1,2,\cdots, l\right\},
\]
where \(W_{(j)} \in \mathbb{R}^{k_jk_jd_{j-1} \times d_j} \) represents the matrix obtained by vectorizing the convolution kernels \(\Theta^1_{(j)},\Theta^2_{(j)}, \cdots, \Theta^{d_j}_{(j)}\) at the \(j\)-th layer. Specific details are provided in \ref{appendixA}. \(\lVert \cdot \rVert_\sigma\) denotes the spectral norm. Note that in formalizing the hypothesis space \(\mathbb{F}_l\), we assume the weight matrix has a bounded spectral norm for each layer. This assumption is widely employed in the literature for the theoretical analysis of neural networks \citep{he2020recent}. Considering that the final layer of a CNNs is typically a fully connected layer, we further define the complete hypothesis space as
\[
\mathbb{F} =  \left\{\sigma_{L+1}(W_{(L+1)}\text{vec}(g)): g\in \mathbb{F}_{L}, \lVert W_{(L+1)} \rVert_\sigma \leq b_{L+1}\right\},
\]
where \(\text{vec}(\cdot)\) represents the vectorization operation. For example, for a matrix \(A \in \mathbb{R}^{s\times t}\), we have \(\text{vec}(A) = (A(:, 1)^\top, A(:, 2)^\top, \cdots, A(:, t)^\top) \in \mathbb{R}^{st}\), and for a tensor \(B \in \mathbb{R}^{s \times t \times q} \), we have \(\text{vec}(B) = (\text{vec}(B(:,:, 1))^\top, \text{vec}(B(:,:, 2))^\top, \cdots, \text{vec}(B(:,:, q))^\top) \in \mathbb{R}^{stq}\).
 
Before proceeding further, leveraging the ramp loss function and the defined function space of CNNs, we introduce a new space defined as
\[
\mathcal{H}_\gamma = \left\{(R, Y) \mapsto \ell_\gamma(-\mathcal{M}(f(R), Y)): f\in \mathbb{F}\right\}.
\]

Clearly, \(\mathcal{H}_\gamma\) utilizes functions from \(\mathbb{F}\) to transform input images and their corresponding labels into respective losses. In this paper, we establish theoretical results based on the global Rademacher complexity, defined as
\[
\Re(\mathcal{H}_\gamma|_I) = \mathbb{E}_{I}\left[\widehat{\Re}(\mathcal{H}_\gamma|_I)\right],\ \widehat{\Re}(\mathcal{H}_\gamma|_{I}) = \mathbb{E}_u \left[\sup_{f\in \mathbb{F}}
\frac{1}{n} \sum_{i=1}^{n}u_i\ell_\gamma(-\mathcal{M}(f(R_i), Y_i))\right], ,
\]
where \(\widehat{\Re}(\mathcal{H}_\gamma|_{I})\) is the empirical counterpart of \(\Re(\mathcal{H}_\gamma|_I)\), \(u = (u_1, u_2,\cdots, u_n)\), \(u_1, u_2,\cdots, u_n\) are independent and identically distributed (i.i.d.) Rademacher variables. Based on the above notations, we first establish the following relationship between margin-based generalization error and Rademacher complexity.

\begin{lemma}\label{generalizationRademacher}
Given dataset \(I = \{(R_i, Y_i)\}^n_{i=1}\), and any margin \(\gamma > 0\), for any \(\delta \in (0, 1)\), with probability at least \(1-\delta\), every \(f \in \mathbb{F}\) satisfies
\[
\mathbb{P}\left(\argmax_{j}f(R; \Theta)_j \neq Y\right) -  \widehat{\mathcal{R}}(f) \leq  2\Re(\mathcal{H}_\gamma|_I) + \sqrt{\frac{\log (1/\delta)}{2n}}
\]
and
\[
\mathbb{P}\left(\argmax_{j}f(R; \Theta)_j \neq Y\right) -  \widehat{\mathcal{R}}(f) \leq  2  \widehat{\Re}(\mathcal{H}_\gamma|_{I})+ 3\sqrt{\frac{\log (2/\delta)}{2n}}.
\]
\end{lemma}

Lemma \ref{generalizationRademacher} suggests that the margin-based generalization error can be bounded by the global Rademacher complexity or its empirical counterpart. In line with our proof strategy, we need to derive an upper bound for the Rademacher complexity or its empirical counterpart. Fortunately, this result has already been established in the literature \citep{bartlett2017spectrally}. Below is a conclusion from \citet{bartlett2017spectrally} that provides the desired bound.

\begin{lemma}\label{Rademachercovering}
Suppose \(0\in \mathcal{H}_\gamma\), and all conditions in Lemma \ref{generalizationRademacher} holds. Then
\[
\Re(\mathcal{H}_\gamma|_{I}) \leq \inf_{\varphi>0} \left(\frac{4\varphi}{n} + \frac{12}{n}\int_{\varphi}^{\sqrt{n}}\sqrt{
\log \mathcal{N}(\mathcal{H}_\gamma|_I, \epsilon, \lVert \cdot\rVert_2)}\right).
\]
\end{lemma}

In Lemma \ref{Rademachercovering}, \(\mathcal{N}(\mathcal{H}_\gamma|_I, \epsilon, \lVert \cdot\rVert_2)\) denotes the cover number of the function class \(\mathcal{H}_\gamma\) with respect to the distance metric \(\lVert \cdot\rVert_2\) and a radius of \(\epsilon\). By combining the first conclusion of Lemma \ref{generalizationRademacher} and Lemma \ref{Rademachercovering}, it suggests that the generalization error bound can be easily obtained if the covering number of the function class \(\mathcal{H}_\gamma\) can be determined. In the following, we derive the covering number of the function class \(\mathcal{H}_\gamma\) in a recursive form, starting from each layer of the CNNs, then extending to the entire CNNs, and finally encompassing \(\mathcal{H}_\gamma\). Specifically, the following theorem provides an upper bound on the covering number of the function class corresponding to the first layer of CNNs. 

\begin{theorem} \label{singlayercov}
The matrices \(W_{(1)} \in \mathbb{R}^{nk_1k_1 \times d_1}\) and \(F_{(0)} \in  \mathbb{R}^{(M_0 -k_1 + 1)(N_0 - k_1 + 1) \times nk_1 k_1}\) are rearrangements of \(\Theta_{(1)}\) and \(Z_{(0)}\) based on equations (1) and (2) in Supplementary Material, respectively. Then
\[
\log \mathcal{N}\left(\left\{\Theta_{(1)} \circledast Z_{(0)} , \lVert W_{(1)} \rVert_2 \leq b_1\right\}, \epsilon, \lVert \cdot \rVert_2\right) \leq \bigg\lceil \frac{\lVert F_{(0)} \rVert_2^2b_1^2d_1}{\epsilon^2}\bigg \rceil\ln(2nk_1k_1d_1).
\]
\end{theorem}

Based on Theorem \ref{singlayercov}, we can obtain an upper bound on the covering number of the function class corresponding to the entire CNNs {by employing a recursive approach}.
\begin{theorem}\label{coveringbound}
Assume that the non-linear activations of the \(i\)th layer is \(\eta_{i}\)-Lipschitz continuous and the number of reuses of each element in the convolution operation of each layer is \(m_i, i = 1,2,\cdots, m_L\) and \(m_{L+1} = 1\). Then the covering number of CNNs class \(\mathbb{F}\) satisfies
\[
\ln \mathcal{N}(\mathbb{F}, \epsilon, \lVert \cdot \rVert_2) \leq \ln\left(2C\widetilde{M}_{L} \widetilde{N}_{L} d_{L} \vee \max_{i \leq L} 2d_ik^2_i\right)  \frac{\lVert Z^k_{(0)} \rVert^2_2 \prod_{j\leq L+1} m_j \eta^2_{j} b^2_j}{\epsilon^2}(L+1)^3.
\]
\end{theorem}

From Theorems \ref{singlayercov} and \ref{coveringbound}, it is evident that there is a significant relationship between the upper bound of the covering number and the structure of CNNs, as well as the non-linear activation functions used in each layer. By utilizing this upper bound of the covering number and combining it with Lemmas \ref{generalizationRademacher} and \ref{Rademachercovering}, we can derive the following margin-based generalization error bound that is directly related to the structure of CNNs.

\begin{theorem}\label{finaltheorem}
Assume that the conditions in Lemma \ref{Rademachercovering} and Theorem \ref{coveringbound} both hold. Then, for any \(\delta \in (0, 1)\), with probability of at least \(1-\delta\), every \(f \in \mathbb{F}\) satisfies
\[
\mathbb{P} \left(\argmax_{i}f(R)_i \neq Y\right)-  \widehat{\mathcal{R}}_\gamma(f) \leq  \frac{24\sqrt{Q}}{n}\left(1 + \ln\left(n/(3\sqrt{Q})\right)\right)+ \sqrt{\frac{\log (1/\delta)}{2n}},
\]
where \(Q = \ln\left(2C\widetilde{M}_{L} \widetilde{N}_{L} d_{L} \vee \max_{i \leq L} 2d_ik^2_i\right)  \frac{4\lVert \mZ^k_{(0)} \rVert^2_2 \prod_{j\leq L+1} m_j \eta^2_{j} b^2_j}{\gamma^2}(L+1)^3\).
\end{theorem}

In deriving the aforementioned generalization error bound, we have incorporated assumptions that are widely used in existing literature \citep{bartlett2017spectrally}. This implies that the conclusion has a certain degree of generality, and under more rigorous assumptions, it is possible to obtain sharper bounds. The generalization error bound in Theorem \ref{finaltheorem} is applicable to any convex combination coefficients and network parameters, and therefore, it also applies to the \(\alpha^*\) and \(\Theta^*\) obtained based on the ADMM algorithm. Furthermore, it should be noted that the derived generalization error bound is specific to a particular network structure.

\section{Simulation Studies} \label{sec5simulate}
The purpose of this section is to assess the performance of the proposed method through simulation studies, while simultaneously investigating the impact of classifying using point-based and interval-based imaging methods.

\subsection{Simulation design} \label{sec51}
In the simulation studies, we investigated three distinct data generation processes (DGPs). For each of these processes, we first generate the center and range, and then reconstruct the upper and lower bounds. We assume that the central and range residuals follow a bivariate normal distribution, denoted as \(\epsilon_t = (\epsilon^c_t, \epsilon^r_t)^\top \sim \mathcal{N}(0, \Sigma)\), where \(\Sigma = \left(\begin{array}{cc}
	1 & \rho/2\\
	\rho/ 2 & 1/4
\end{array}\right)\), and \(\rho\) signifies the correlation coefficient. Moreover, we assign values to \(\rho\) as -0.9, -0.5, 0, 0.3, and 0.7, respectively. The subsequent outline the specifics of the three DGPs under consideration. DGP1:
\[
x_t = (x^c_t, x^r_t)^\top = \sum_{l=1}^{\infty} \pi_l \Phi z_{t,l} + \epsilon_{t},
\]
where \(z_{t,1}= (1, 1)^\top\), \(z_{t, l} \sim \mathcal{N}(0, \Sigma), l\geq 2\), and \(\Phi = \left(\begin{array}{cc}
	0.2 & -0.1\\
	0.1 & 0.2
\end{array}\right)\), \(\pi_l = l^{-2} / \sqrt{3}\). DGP2:
\[ 
x_t = (x^c_t, x^r_t)^\top =\Phi  x_{t-1} + \epsilon_{t} - \Gamma \epsilon_{t-1}, 
\]
where \(\Phi\) is the same as in DGP1, and \(\Gamma=\left(\begin{array}{cc}
	-0.6 & 0.3\\
	0.3 & 0.6
\end{array}\right)\). DGP3: 
\[
x_t = (x^c_t, x^r_t)^\top =\epsilon_t - \Gamma \epsilon_{t-1},
\]
where \(\Gamma\) is the same as in DGP2.

\subsection{Univariate interval-valued time series classification}\label{ss1}
For each DGP, the interval-valued time series data generated by different correlation coefficients \(\rho\) is divided into distinct classes, {in which} each class containing 500 samples. The length of interval-valued time series generated by all DGPs is \(T=150\). As a result, we obtain three interval-valued time series datasets {and each consists} of five classes with a total of 2500 samples. For each dataset, we randomly selected 80\% of the data as the training set and the remaining as the validation set. We set the threshold used in the RP imaging method to \(\pi/18\). From Figure \ref{app}, it is evident that different \(\nu\) result in varying levels of approximation from \((1 + \tanh(\nu x)) / 2\) to the Heaviside function \(H(x)\). Therefore, we select five distinct \(\nu\): 1, 5, 10, 15, and 20, to assess the impact of approximation levels on classification performance. Throughout all the experiments, we utilize classification accuracy as the evaluation metric.

As an illustrative example, Figure \ref{DGP1iaageRP} depict the images generated using RP based on DGP1. It is evident that interval-valued time series from distinct classes (distinguished by varying correlation coefficients \(\rho\)) yield highly akin images under the same level of smoothness. Nevertheless, relying solely on a single CNNs, such as ResNet \cite{He2016}, would not yield satisfactory classification outcomes. To address this limitation, we employ the fine-grained image classification network WS-DAN \cite{DBLP:journals/corr/abs-1901-09891} to classify the imaged dataset. In our comparative analysis with other methods, considering that existing techniques represent intervals using representative points for classification, we employ diverse representations, including center (\(c\)), range (\(r\)), upper bound (\(u\)), and lower bound (\(l\)), to achieve effective classification for interval-valued data.

The classification performance of the model based on the imaging method RP under different approximation levels across the three DGPs is summarized in Table \ref{unRPIRP}.

\begin{table}[H]
\setlength{\abovecaptionskip}{0pt}%
\setlength{\belowcaptionskip}{3pt}%
\centering
\caption{Classification accuracy based on WS-DAN using RP} \label{unRPIRP}
\setlength{\tabcolsep}{5mm}{
\begin{tabular}{cccccc}
	\hline
	\multirow{2}{*}{\tabincell{c}{DGP}} & \multicolumn{5}{c}{\(\nu\)(RP) } \\
	\cline{2-6}
	&  	\(\nu = 1\) & 	\(\nu = 5\) & 	\(\nu = 10\) & 	\(\nu = 15\) &	\(\nu = 20\)\\
	\hline
	DGP1 & 91.67&81.50&88.98&\textbf{95.83}&91.67\\
	DGP2 & \textbf{85.42}&66.67&75.00&72.92&68.32\\
	DGP3 &80.56&80.83&83.33&\textbf{85.42}&76.28\\
	\hline 
\end{tabular}
}
\end{table}

From Table \ref{unRPIRP}, it is evident that there is no clear trend of improved classification accuracy with increasing approximation levels. For instance, when \(\nu=15\), DGP1 and DGP3 achieved the highest accuracy, while \(\nu=1\), DGP2 attained the best accuracy. This implies that the images directly obtained through the RP method do not necessarily exhibit the highest discriminatory capability.

In the classification process based on the imaging method RP, the convex combination coefficients \(\alpha\) are considered as learnable parameters. The following figures depict the optimal convex combination coefficients obtained under different degrees of smoothing for the three DGPs. These graphs illustrate that relying on simple representative points within intervals for classification might not yield optimal discriminatory results, as the choice of most representative points can vary across different dimensions.

\begin{figure}[H]
\centering
\begin{minipage}[t]{0.19\textwidth}
\centering
\includegraphics[scale=0.4]{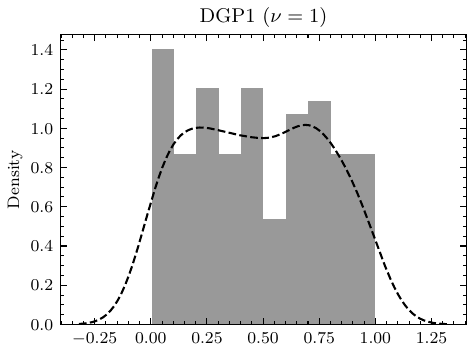}
\end{minipage}
\begin{minipage}[t]{0.19\textwidth}
\centering
\includegraphics[scale=0.4]{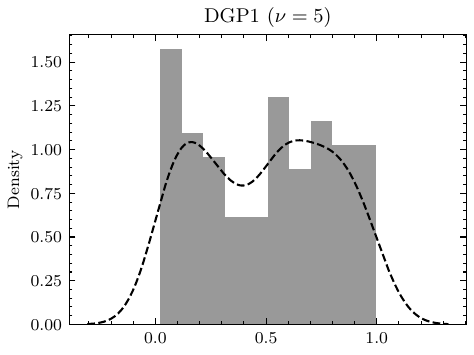}
\end{minipage}
\begin{minipage}[t]{0.19\textwidth}
\centering
\includegraphics[scale=0.4]{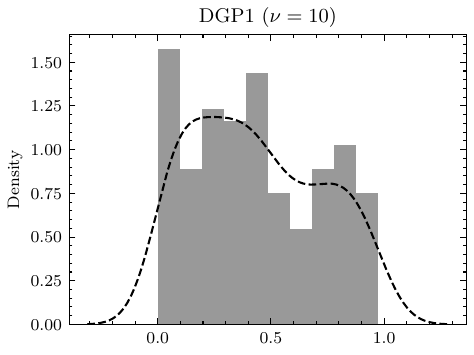}
\end{minipage}
\begin{minipage}[t]{0.19\textwidth}
\centering
\includegraphics[scale=0.4]{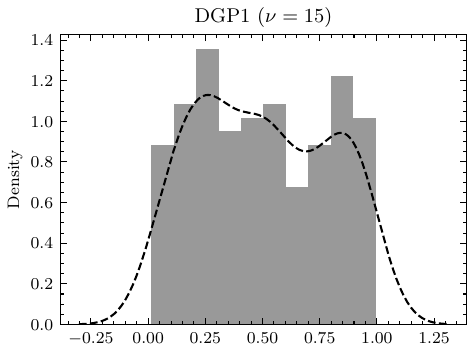}
\end{minipage}
\begin{minipage}[t]{0.19\textwidth}
\centering
\includegraphics[scale=0.4]{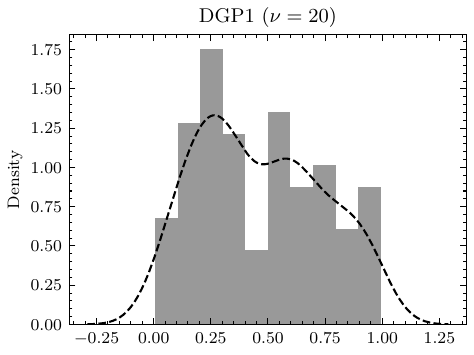}
\end{minipage}

\begin{minipage}[t]{0.19\textwidth}
	\centering
	\includegraphics[scale=0.4]{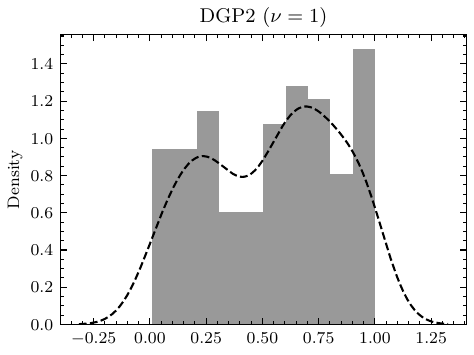}
\end{minipage}
\begin{minipage}[t]{0.19\textwidth}
	\centering
	\includegraphics[scale=0.4]{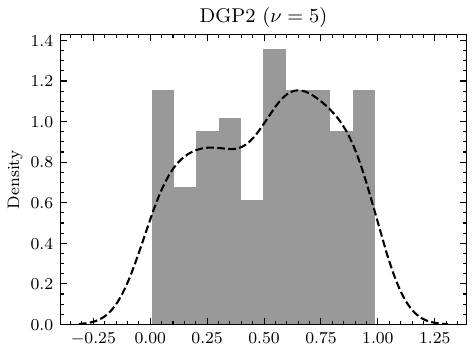}
\end{minipage}
\begin{minipage}[t]{0.19\textwidth}
	\centering
	\includegraphics[scale=0.4]{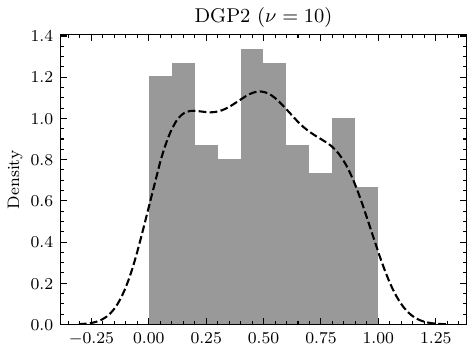}
\end{minipage}
\begin{minipage}[t]{0.19\textwidth}
	\centering
	\includegraphics[scale=0.4]{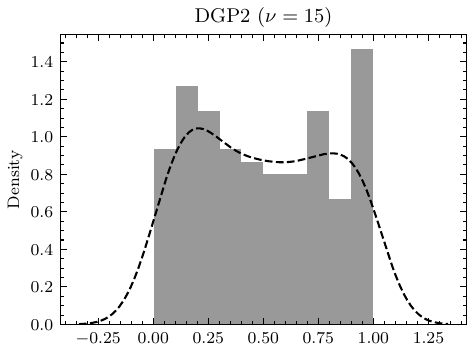}
\end{minipage}
\begin{minipage}[t]{0.19\textwidth}
	\centering
	\includegraphics[scale=0.4]{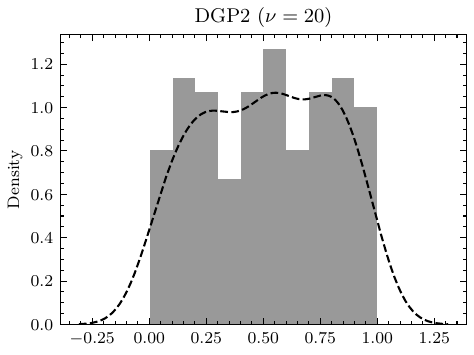}
\end{minipage}

\begin{minipage}[t]{0.19\textwidth}
	\centering
	\includegraphics[scale=0.4]{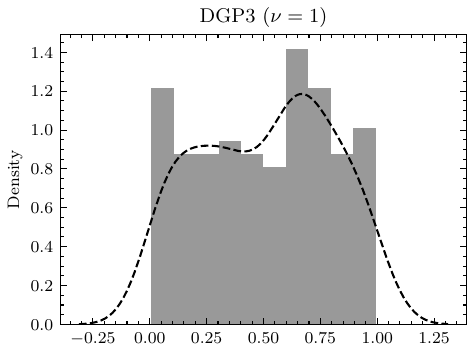}
\end{minipage}
\begin{minipage}[t]{0.19\textwidth}
	\centering
	\includegraphics[scale=0.4]{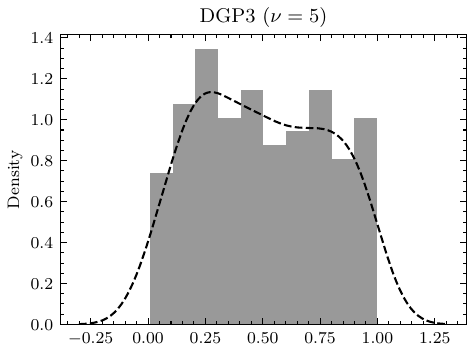}
\end{minipage}
\begin{minipage}[t]{0.19\textwidth}
	\centering
	\includegraphics[scale=0.4]{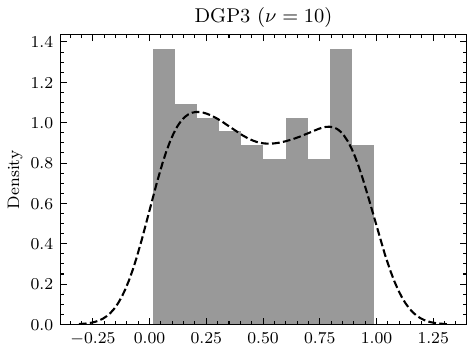}
\end{minipage}
\begin{minipage}[t]{0.19\textwidth}
	\centering
	\includegraphics[scale=0.4]{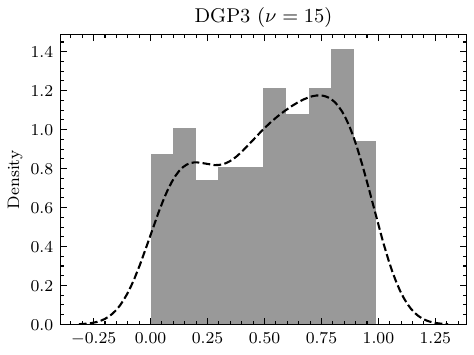}
\end{minipage}
\begin{minipage}[t]{0.19\textwidth}
	\centering
	\includegraphics[scale=0.4]{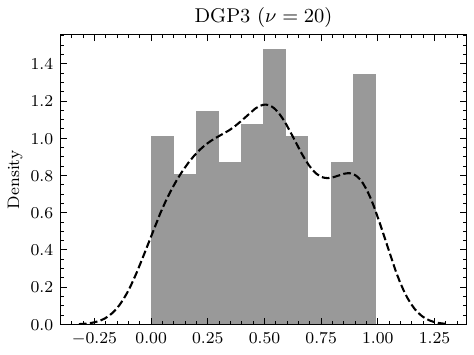}
\end{minipage}

\caption{Distributions of optimal convex combination coefficients obtained under different degrees of smoothing for the three DGPs. Each row represents a specific DGP, and each column represents a level of smoothing.}

\end{figure}

\begin{figure}[H]
	\centering
	\begin{minipage}[t]{0.19\textwidth}
		\centering
		\includegraphics[scale=0.52]{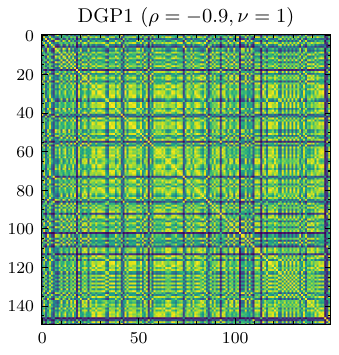}
	\end{minipage}
	\begin{minipage}[t]{0.19\textwidth}
		\centering
		\includegraphics[scale=0.52]{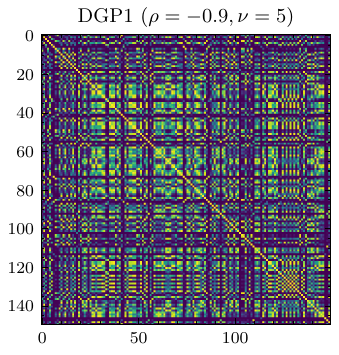}
	\end{minipage}
	\begin{minipage}[t]{0.19\textwidth}
		\centering
		\includegraphics[scale=0.52]{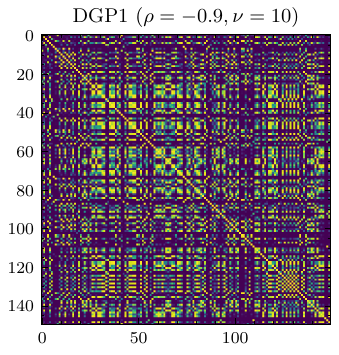}
	\end{minipage}
	\begin{minipage}[t]{0.19\textwidth}
		\centering
		\includegraphics[scale=0.52]{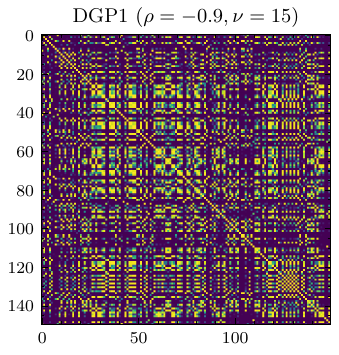}
	\end{minipage}
	\begin{minipage}[t]{0.19\textwidth}
		\centering
		\includegraphics[scale=0.52]{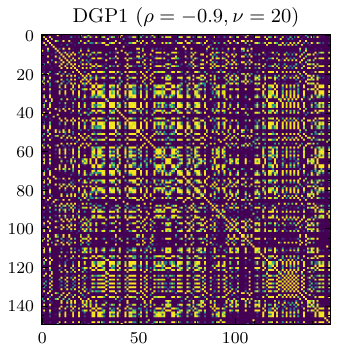}
	\end{minipage}
	
	\begin{minipage}[t]{0.19\textwidth}
		\centering
		\includegraphics[scale=0.52]{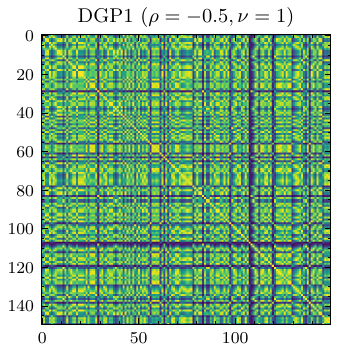}
	\end{minipage}
	\begin{minipage}[t]{0.19\textwidth}
		\centering
		\includegraphics[scale=0.52]{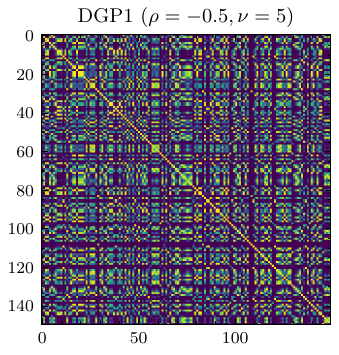}
	\end{minipage}
	\begin{minipage}[t]{0.19\textwidth}
		\centering
		\includegraphics[scale=0.52]{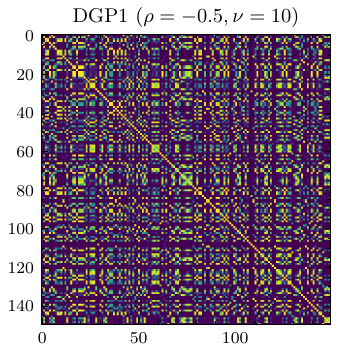}
	\end{minipage}
	\begin{minipage}[t]{0.19\textwidth}
		\centering
		\includegraphics[scale=0.52]{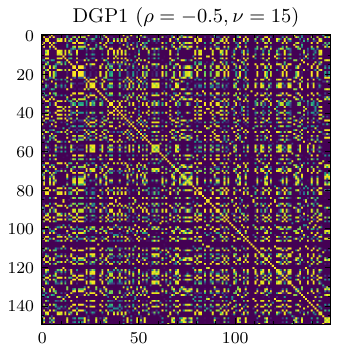}
	\end{minipage}
	\begin{minipage}[t]{0.19\textwidth}
		\centering
		\includegraphics[scale=0.52]{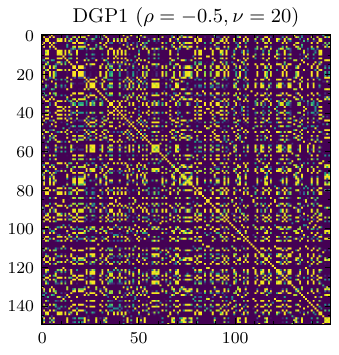}
	\end{minipage}
	
	\begin{minipage}[t]{0.19\textwidth}
		\centering
		\includegraphics[scale=0.52]{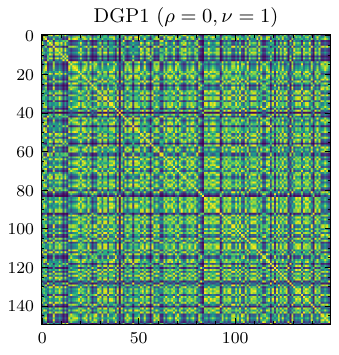}
	\end{minipage}
	\begin{minipage}[t]{0.19\textwidth}
		\centering
		\includegraphics[scale=0.52]{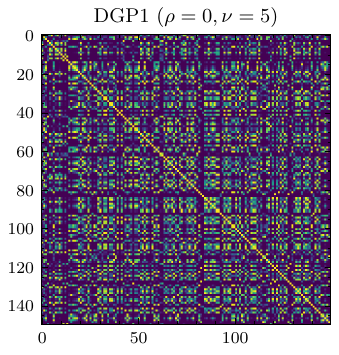}
	\end{minipage}
	\begin{minipage}[t]{0.19\textwidth}
		\centering
		\includegraphics[scale=0.52]{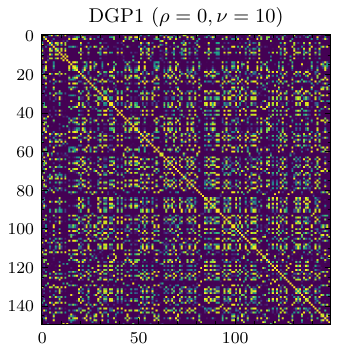}
	\end{minipage}
	\begin{minipage}[t]{0.19\textwidth}
		\centering
		\includegraphics[scale=0.52]{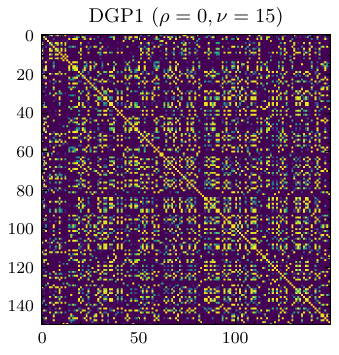}
	\end{minipage}
	\begin{minipage}[t]{0.19\textwidth}
		\centering
		\includegraphics[scale=0.52]{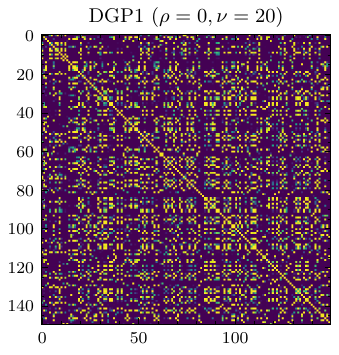}
	\end{minipage}
	
	\begin{minipage}[t]{0.19\textwidth}
		\centering
		\includegraphics[scale=0.52]{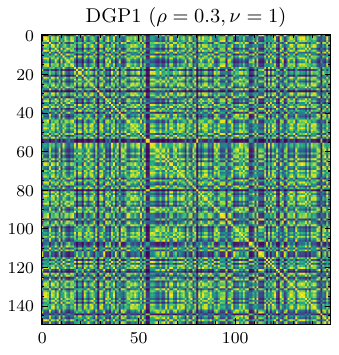}
	\end{minipage}
	\begin{minipage}[t]{0.19\textwidth}
		\centering
		\includegraphics[scale=0.52]{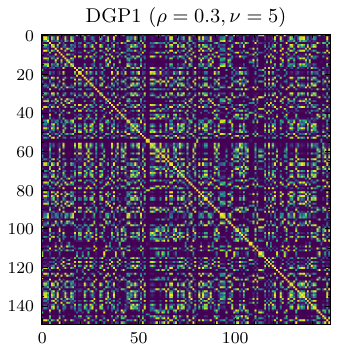}
	\end{minipage}
	\begin{minipage}[t]{0.19\textwidth}
		\centering
		\includegraphics[scale=0.52]{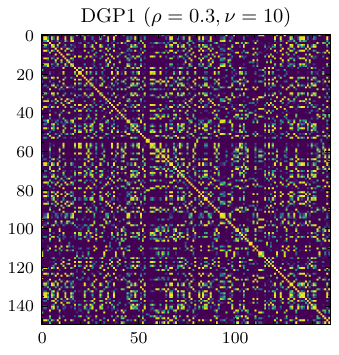}
	\end{minipage}
	\begin{minipage}[t]{0.19\textwidth}
		\centering
		\includegraphics[scale=0.52]{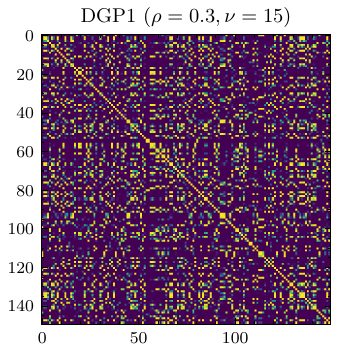}
	\end{minipage}
	\begin{minipage}[t]{0.19\textwidth}
		\centering
		\includegraphics[scale=0.52]{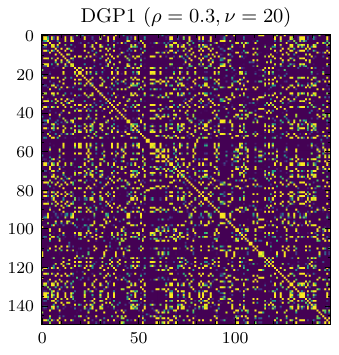}
	\end{minipage}

\begin{minipage}[t]{0.19\textwidth}
	\centering
	\includegraphics[scale=0.52]{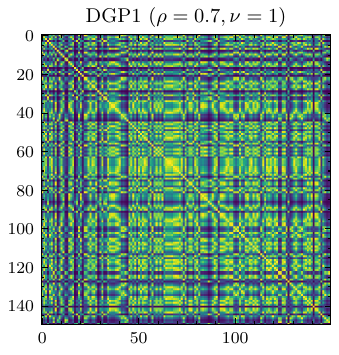}
\end{minipage}
\begin{minipage}[t]{0.19\textwidth}
	\centering
	\includegraphics[scale=0.52]{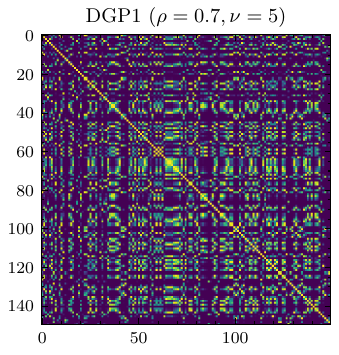}
\end{minipage}
\begin{minipage}[t]{0.19\textwidth}
	\centering
	\includegraphics[scale=0.52]{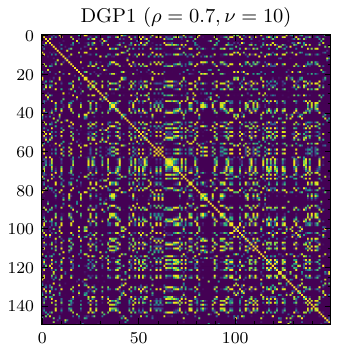}
\end{minipage}
\begin{minipage}[t]{0.19\textwidth}
	\centering
	\includegraphics[scale=0.52]{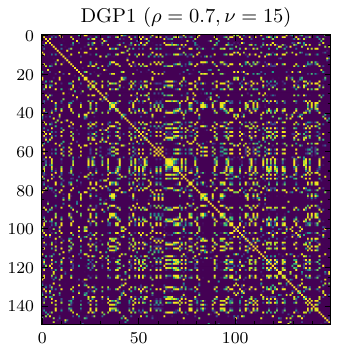}
\end{minipage}
\begin{minipage}[t]{0.19\textwidth}
	\centering
	\includegraphics[scale=0.52]{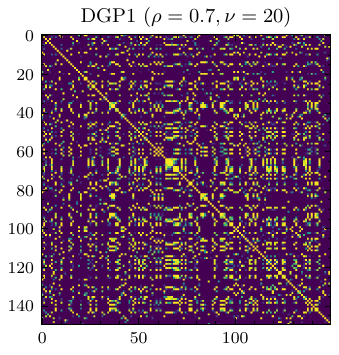}
\end{minipage}	
\caption{Images obtained based on RP corresponding to DGP1 at various correlation coefficients and levels of approximation. Each row represents the images obtained at different levels of approximation for the same correlation coefficient (class), while each column indicates the images for different classes at the same level of approximation.}\label{DGP1iaageRP}

\end{figure}

We compare {the} proposed method with all five types of methods discussed in Section \ref{sec1}. {KNeighbors (KN) and ShapeDTW (SD) \cite{2017shapeDTW} were selected among the distance-based
methods, and IndividualBOSS (IB), and IndividualTDE (ITDE) were selected among
the dictionary-based methods. In addition, the MrSQM \cite{Nguyen2021MrSQMFT} method was selected among
the shapelet-based methods by considering the computation time, and Bagging,
ComposableTimeSeriesForest (CTF) \cite{DENG2013142} and WeightedEnsemble (WE) were selected
among the ensemble learning methods.}

The ongoing advancements in deep learning have given rise to a multitude of models and techniques for time series classification. In this study, we meticulously handpicked 16 representative methods from the state-of-the-art deep learning-based time series analysis framework \texttt{tsai} for the purpose of comparison. The relevant references and source code for these methods can be accessed at \url{https://timeseriesai.github.io/tsai/}. The rationale behind incorporating a diverse range of ensemble learning and deep learning methods in the aforementioned approaches stems from their remarkable performance across various scenarios in contemporary research. The ensuing table presents the classification accuracy achieved by the five types of methods at four representative points within the intervals.
\begin{table}[H]
\small
\setlength{\abovecaptionskip}{0pt}%
\setlength{\belowcaptionskip}{3pt}%
\centering
\caption{Classification accuracy based on point-value time series methods.} \label{unpointmeth}
\setlength{\tabcolsep}{1mm}{ 
\begin{tabular}{l|cccc|cccc|cccc}
\hline
\multirow{2}{*}{Methods} & \multicolumn{4}{c|}{DGP1} & \multicolumn{4}{c|}{DGP2} & \multicolumn{4}{c}{DGP3}\\
&  \(c\)&\(r\)&\(l\)& \(u\) & \(c\)&\(r\)&\(l\)& \(u\) & \(c\)&\(r\)&\(l\)& \(u\)\\
\hline
KN &19.60&21.40&18.40&20.40&20.60&24.40&18.20&22.60&18.20&22.00&19.60&21.00\\
SD &19.00&19.20&18.34&20.40&19.80&25.60&19.20&23.60&21.80&21.20&19.00&21.20\\
IB &23.60&17.60&21.00&30.80&20.00&23.60&18.80&24.20&19.60&26.00&16.00&24.40\\
ITDE& 19.00&20.80&19.60&23.80&19.60&33.20&21.80&27.60&19.00&32.80&21.80&24.80\\
MrSQM &22.00&18.20&21.60&21.20&21.20&50.80&19.80&44.40&25.80&62.40&18.80&48.40\\
Bagging& 22.40&16.80&21.40&18.60&34.00&42.80&20.60&41.80&27.40&38.80&19.60&43.60\\
CFT &21.60&19.00&82.00&\textbf{86.20}&30.80&45.00&54.00&45.20&31.40&60.60&49.40&65.80\\
WE & 20.20&20.60&19.80&17.20&26.40&38.80&19.80&42.40&25.40&35.60&15.80&43.60\\
\hline 
FCN &22.25 & 24.75 & 83.00& 83.25& 51.75 & 65.25& 64.75& 73.75& 46.00& 71.00& 53.00& 75.25\\
FCNPlus & 21.75&23.50&83.00&84.25&52.50&65.00&64.50&73.75&45.75&72.25&55.25&74.75\\
Inception &22.25&23.75&77.25&80.75&47.75&57.25&58.50&70.25&40.50&64.00&49.25&71.25\\
XCoord & 22.75&21.50&83.75&83.50&49.75&67.00&62.50&73.00&47.00&\textbf{72.75}&54.50&76.00\\
MLP &24.00&23.00&44.50&50.75&23.25&31.00&31.50&29.00&22.00&38.25&23.75&40.50\\
RNN& 22.50&19.75&24.00&28.00&22.50&23.25&22.00&23.25&19.75&22.50&23.25&21.25\\
LSTM &  20.25&20.25&46.00&46.25&32.50&29.75&36.25&41.75&24.00&43.00&37.75&67.00\\
GRU & 21.50&20.25&84.25&85.25&30.00&63.00&44.25&56.25&29.25&64.00&25.50&78.25\\
RNN\_FCN & 22.25&23.75&82.50&81.25&51.25&66.00&62.50&73.00&42.00&67.75&51.75&74.75\\
LSTM\_FCN & 22.25&23.75&81.00&78.50&48.50&62.25&59.75&72.50&48.00&68.75&53.75&\textbf{79.25}\\
GRU\_FCN  & 21.50&23.75&82.00&82.75&50.75&65.00&62.25&71.50&41.50&66.00&53.25&71.00\\
MRNN\_FCN &22.75&24.00&84.00&83.25&50.50&62.75&64.50&72.75&45.25&70.75&51.25&74.75\\
MLSTM\_FCN &23.25&23.75&\textbf{85.75}&80.50&50.00&63.75&62.50&72.75&44.00&69.50&52.50&77.25\\
MGRU\_FCN & 22.25&24.00&82.75&84.75&49.75&64.50&64.25&\textbf{76.00}&42.50&69.25&48.75&72.75\\
ResCNN &24.00&24.00&80.00&82.50&48.50&61.25&61.00&71.00&43.50&67.50&53.50&77.00\\
ResNet & 22.25&22.50&77.25&80.50&48.75&61.00&59.00&67.00&39.00&65.50&49.25&68.50\\
\hline
\end{tabular}
}
\end{table}

From Table \ref{unpointmeth}, it is apparent that both these deep learning and non-deep learning methods exhibit poor performance across the three DGPs. For instance, the optimal performance for DGP1, DGP2, and DGP3 is 86.20\%, 76\%, and 79.25\% respectively, with the majority of the remaining accuracies not even surpassing fifty percent. The primary reason for this might be attributed to our approach considering the interval-valued time series generated with different correlation coefficients as distinct classes, where the data among various classes share significant similarities.

Comparing Table \ref{unRPIRP} and Table \ref{unpointmeth}, it is clearly evident that the accuracy of discriminating interval-valued time series by first transforming them into images and then performing classification surpasses the accuracy achieved by classifying based on representative points. For instance, in the case of DGP1, the optimal accuracy achieved using the imaging method RP is 95\% at \(\nu=15\), while the optimal accuracy using the representative points method is only 86.20\%. Moreover, the classification accuracy remains relatively high for other values of {\(\nu\)} as well. This suggests that simply using representative points within intervals for interval representation and classification is insufficient, as the most crucial point representations vary across different dimensions. Furthermore, the ineffectiveness of most deep learning methods on these DGPs further highlights the superiority of the proposed approach in this paper.

\subsection{Multivariate interval-valued time series classification} \label{sec53}
Using the DGPs of univariate interval-valued time series in Section \ref{sec51}, we construct the following two multivariate interval-valued time series generation processes. In the first scenario (C1), each DGP is treated as a class, and the interval-valued time series generated based on different correlation coefficients \(\rho\) are considered as different dimensions. This approach results in a multivariate interval-valued time series dataset consisting of three classes and five dimensions. In the second scenario (C2), the interval-valued time series generated based on different correlation coefficients \(\rho\) are treated as different classes, and each DGP is considered as a dimension. This setup yields a multivariate interval-valued time series dataset consisting of five classes and three dimensions. The subsequent table presents the classification performance of multivariate interval-valued time series based on JRP for the two scenarios.

\begin{table}[H]
\setlength{\abovecaptionskip}{0pt}%
\setlength{\belowcaptionskip}{3pt}%
\centering
\caption{Classification accuracy based on WS-DAN using JRP} \label{siJRPIJRP}
\setlength{\tabcolsep}{5mm}{
\begin{tabular}{cccccc}
\hline
\multirow{2}{*}{\tabincell{c}{DGP}} & \multicolumn{5}{c}{\(\nu\)(JRP) }\\
\cline{2-6}
&  	\(\nu = 1\) & 	\(\nu = 5\) & 	\(\nu = 10\) & 	\(\nu = 15\) &	\(\nu = 20\)\\
\hline
C1 & 88.33&89.44&95.67&\textbf{96.43}&88.34\\
C2 &86.72&89.58&\textbf{97.22}&92.36&95.83\\
\hline
\end{tabular}
}
\end{table}

Similar to univariate interval-valued time series, we select five types of multivariate point-value time series Classification methods for comparison. However, due to the inapplicability of certain methods to multivariate data, their corresponding results are not presented in the table below.

\begin{table}[H]
\setlength{\abovecaptionskip}{0pt}%
\setlength{\belowcaptionskip}{3pt}%
\centering
\caption{Classification accuracy based on point-value time series methods.} \label{simumuipoint}
\setlength{\tabcolsep}{1mm}{ 
\begin{tabular}{l|cccc|cccc}
\hline
\multirow{2}{*}{Methods} & \multicolumn{4}{c|}{C1} & \multicolumn{4}{c}{C2}\\
&  \(c\)&\(r\)&\(l\)& \(u\) & \(c\)&\(r\)&\(l\)& \(u\) \\
\hline
KN& 37.33&38.67&37.33&37.33&17.00&20.80&18.20&20.40\\
ITDE & 71.00  &70.33&65.67&73.00&18.60&38.20&18.60&30.20\\
Bagging & 82.33&75.67&76.67&68.67&23.20&37.40&21.60&35.80\\
WE& 77.00&69.33&73.67&71.33&21.00&34.80&19.60&41.60\\

\hline 
FCN & 98.75&\textbf{99.5}8&98.75&\textbf{100.0}&51.50&77.00&\textbf{88.50}&\textbf{90.75} \\
Inception & \textbf{99.1}7&99.17&\textbf{99.58}&98.75&52.00&75.50&85.25&86.75\\
XCoord & 97.50&98.33&96.25&\textbf{100.0}&54.50&\textbf{77.50}&84.50&89.25\\
MLP & 37.92&69.17&50.83&66.67&23.75&26.00&23.50&33.75\\
RNN& 66.67&74.17&48.33&65.83&19.75&22.50&24.25&23.25\\
LSTM & 75.42&86.67&63.33&96.25&31.25&43.00&60.50&68.75\\
GRU & 88.75&96.67&74.58&93.75&36.25&75.25&87.25&89.25\\
RNN\_FCN & 97.50&98.75&96.67&98.33&52.25&70.50&77.75&80.50\\
LSTM\_FCN & 94.58&99.58&94.17&99.58&50.00&72.00&80.00&82.25\\
GRU\_FCN  & 95.83&95.83&92.50&99.58&50.50&68.50&80.00&79.00\\
MRNN\_FCN &98.75&97.50&97.50&98.33&\textbf{55.25}&74.75&83.75&84.25\\
MLSTM\_FCN &91.67&96.67&97.50&98.33&54.50&76.00&84.00&84.75\\
MGRU\_FCN &98.75&96.67&93.33&98.75&50.00&77.00&83.25&83.50\\
ResCNN & 97.08&95.00&95.83&99.58&46.75&73.00&84.75&87.75\\
ResNet & 98.33&95.00&95.83&94.58&48.00&71.50&82.00&87.75\\
\hline
\end{tabular}
}
\end{table}

From Table \ref{siJRPIJRP}, it can be observed that the approximation level does not appear to be the sole determinant of the final accuracy of the model. This finding is consistent with the conclusions drawn from the univariate interval-valued time series classification. In both scenarios C1 and C2, the optimal accuracies based on the JRP method are 96.43\% and 97.22\%, respectively. 

From Table \ref{simumuipoint}, it can be observed that non-deep learning methods are almost ineffective for both scenarios C1 and C2 and exhibit minimal classification capability. On the other hand, deep learning-based methods exhibit strong classification performance in C1. For instance, when the representative point is chosen as the upper bound of the interval, the methods like FCN and XCoord achieve perfect accuracy, surpassing the performance of the proposed methods. However, these deep learning-based methods do not perform well in C2. For instance, when the interval center is chosen as the representative point, the accuracy is only around 50\%. This result can be easily explained, as C2 involves categorizing interval-valued time series generated from different correlation coefficients into distinct classes, resulting in high similarity between data from different classes and increasing the difficulty of classification. Comparing the results from Tables \ref{siJRPIJRP} and \ref{simumuipoint}, it can be observed that regardless of the complexity of the classification task, the proposed methods consistently achieve a satisfactory level of classification accuracy. In contrast, both existing deep learning and non-deep learning methods tend to almost fail in certain complex scenarios. This highlights the universality and versatility of the proposed methods.

\section{Real Data Examples}\label{sec6rreal}
Some existing methods \cite{Utkin2016, Wang2017} for classifying interval-valued time series in real data studies adopt an approach that involves converting real point-valued time series into interval-valued time series through a fuzzification process. However, the drawback of these methods is that they introduce certain artificial design and random elements during the interval construction, which makes it difficult to validate the efficacy of the proposed methods using real data. In this section, we validate the effectiveness of our proposed method using naturally occurring interval-valued time series.

We utilize the real-world data obtained from the website \url{https://rp5.ru/}. The website encompasses a comprehensive repository of weather data spanning 241 countries and regions across the globe, capturing records from the earliest available periods. Additionally, the website offers weather forecasts and insights for the upcoming seven days. The database receives the latest data very three hours and is updated eight times per day, and records a total of 28 different weather indicators. From this array, we specifically selected 6 indicators for subsequent classification. These selected indicators include atmospheric temperature at a height of two meters above ground level (measured in degrees Celsius), atmospheric pressure at the weather station level (measured in mmHg), atmospheric pressure at mean sea level (measured in mmHg), relative humidity at a height of two meters above ground level (measured in \%), average wind speed at a height of 10-12 meters above ground level during the ten minutes preceding observation (measured in m/s), and dew point temperature at a height of two meters above ground level (measured in degrees Celsius).

We transform the point-valued time series into interval-valued time series on a daily basis, where each interval spans one day. For a given day, we extract the maximum and minimum values of the observations for each indicator, and use them as as the upper and lower bounds of the corresponding interval, respectively. Upon obtaining the interval-valued time series, we consolidate 30 days' worth of observations into a trajectory, which formed the basis for the classification.

\subsection{Univariate interval-valued time series classification} \label{sec61}
We designate four distinct classes based on geographical locations within China: Harbin (H) and Beijing (B) in the north, Taiyuan (T) in the central region, and Sanya (S) in the south. Our classification analysis involves three specific combinations: H-S, S-T-B, and H-S-T-B. In the classification of univariate interval-valued time series, we construct the classification models for the aforementioned three combinations based on the temperature indicator among the six selected indicators. The choice of the smoothing parameter \(\nu\), and the classification model structure follows the procedure outlined in Section \ref{sec5simulate}. The corresponding classification results for the three combinations are presented in the following table.

\begin{table}[H]
\setlength{\abovecaptionskip}{0pt}%
\setlength{\belowcaptionskip}{3pt}%
\centering
\caption{Classification accuracy based on WS-DAN using RP} \label{realRPIRP}
\setlength{\tabcolsep}{5mm}{
\begin{tabular}{cccccc}
\hline
\multirow{2}{*}{\tabincell{c}{DGP}} & \multicolumn{5}{c}{\(\nu\)(RP) }\\
\cline{2-6}
&  	\(\nu = 1\) & 	\(\nu = 5\) & 	\(\nu = 10\) & 	\(\nu = 15\) &	\(\nu = 20\)\\
H-S & 83.33&91.66&\textbf{92.22}&91.66&91.25 \\
S-T-B & 85.52&\textbf{95.83}&94.44&91.67&83.33\\
H-S-T-B	& 70.00&66.67&71.21&\textbf{83.33}&75.00\\	
\hline
\end{tabular}
}
\end{table}

\begin{table}[H]
\small
\setlength{\abovecaptionskip}{0pt}%
\setlength{\belowcaptionskip}{3pt}%
\centering
\caption{Classification accuracy based on point-value time series methods.} \label{realunpoint}
\setlength{\tabcolsep}{1mm}{ 
\begin{tabular}{l|cccc|cccc|cccc}
\hline
\multirow{2}{*}{Methods} & \multicolumn{4}{c|}{H-S} & \multicolumn{4}{c|}{S-T-B} & \multicolumn{4}{c}{H-S-T-B}\\
&  \(c\)&\(r\)&\(l\)& \(u\) & \(c\)&\(r\)&\(l\)& \(u\) & \(c\)&\(r\)&\(l\)& \(u\)\\
\hline
KN &46.03&55.56&61.90&53.97&25.56&35.34&32.33&27.82&20.53&27.15&22.52&19.87\\
SD&55.56&58.73&60.32&69.84&33.83&48.12&38.35&24.81&23.84&34.43&25.82&21.19\\
IB&61.90&66.67&65.08&68.25&38.35&35.34&39.85&38.35&26.49&33.11&29.14&34.44\\
ITDE&65.08&69.84&61.9&63.49&37.59&33.08&36.84&36.84&25.83&33.11&28.48&29.14\\
MrSQM&82.54&65.08&74.60&76.19&38.35&43.61&50.38&48.87&43.71&31.13&39.07&36.42\\
Bagging&76.19&80.95&74.60&80.95&39.10&45.86&51.88&50.38&36.42&39.74&42.38&42.38\\
CFT&76.19&87.30&84.13&80.95&41.35&54.14&57.14&46.62&39.07&49.01&50.99&39.07\\
WE&79.37&74.60&88.89&74.60&45.11&40.6&45.11&50.38&35.76&39.07&44.37&38.41\\

\hline 
FCN & \textbf{92.00} & 84.00& 90.00& \textbf{92.00}& 55.66& 59.43& 60.37& 51.89& 52.89& 52.07& 57.02& 55.37\\
FCNPlus & 90.00&86.00&90.00&\textbf{92.00}&55.66&58.49&61.32&59.49&50.41&52.07&57.85&55.37\\ 
Inception & 88.00&86.00&84.00&\textbf{92.00}&59.43&48.11&62.26&51.88&51.24&47.11&51.24&47.11\\
XCoord & 86.00&80.00&86.00&90.00&56.60&53.77&55.66&54.72&51.24&53.72&56.20&55.37\\
MLP & 86.00&76.00&84.00&86.00&56.60&43.39&52.83&45.28&47.10&41.32&46.28&40.50\\
RNN& 84.00&76.00&86.00&86.00&43.40&42.45&44.34&43.40&42.15&41.32&42.98&41.32\\
LSTM & 86.00&76.00&86.00&86.00&41.51&47.17&41.51&43.40&40.50&42.98&38.84&42.98\\
GRU & 86.00&76.00&84.00&86.00&44.34&44.34&44.34&40.57&41.32&42.15&43.80&44.63\\
RNN\_FCN & 90.00&82.00&90.00&\textbf{92.00}&58.49&55.66&59.43&54.72&52.89&52.89&55.37&52.89\\
LSTM\_FCN & \textbf{92.00}&82.00&\textbf{92.00}&\textbf{92.00}&55.66&55.66&57.55&56.60&47.93&52.07&53.72&55.37\\
GRU\_FCN &90.00&84.00&90.00&\textbf{92.00}&54.72&55.66&60.38&56.60&52.07&50.41&55.37&52.07\\
MRNN\_FCN &90.00&84.00&86.00&\textbf{92.00}&55.66&56.60&57.55&56.60&49.59&55.37&53.72&51.24\\
MLSTM\_FCN &90.00&84.00&90.00&\textbf{92.00}&58.49&57.55&60.38&54.72&53.72&53.72&55.37&55.37\\
MGRU\_FCN &90.00&84.00&\textbf{92.00}&\textbf{92.00}&56.60&56.60&60.38&59.43&54.55&51.24&54.55&57.02\\
ResCNN &88.00&\textbf{90.00}&\textbf{92.00}&90.00&59.43&58.49&65.09&60.38&53.72&51.24&58.68&54.55\\
ResNet &90.00&86.00&90.00&90.00&59.43&56.60&61.32&56.60&53.72&52.90&56.20&47.93\\
\hline
\end{tabular}
}
\end{table}

From Table \ref{realRPIRP}, it can be observed that the imaging method RP demonstrates good classification capabilities for the first two combinations. For example, under different approximation levels, the accuracy for combination H-S is around 90\%, and the optimal accuracy for S-T-B is 95.83\% when \(\nu=5\). However, for the third combination H-S-T-B, the maximum accuracy is only 83.33\%, which may be that more categories lead to the
increased complexity of the discriminative task. From Table \ref{realunpoint}, it is evident that these deep learning and non-deep learning methods exhibit classification ability only for combination H-S, in which the best accuracy reaches 92\%. However, this accuracy is still lower than the level achieved by the proposed method. For combinations S-T-B and H-S-T-B, these point-based classification methods are almost ineffective, whereas the proposed method in the paper can achieve an accuracy of 83.33\%. These conclusions align with the findings in Section \ref{ss1} for univariate interval-valued time series classification, which indicates that the proposed method can achieve satisfactory performance even in various complex scenarios.

\subsection{Multivariate interval-valued time series classification} \label{sec62}
In this section, all the modeling settings remain consistent with those described in Ssection \ref{sec61}, with the only variation being the inclusion of all six variables introduced at the beginning of Section \ref{sec6rreal} for classification. The detailed results of the classification are provided below.

\begin{table}[H]
\setlength{\abovecaptionskip}{0pt}%
\setlength{\belowcaptionskip}{3pt}%
\centering
\caption{Classification accuracy based on WS-DAN using JRP} \label{realJRPIJRP}
\setlength{\tabcolsep}{5mm}{
\begin{tabular}{cccccc}
\hline
\multirow{2}{*}{\tabincell{c}{DGP}} & \multicolumn{5}{c}{\(\nu\)(JRP) }\\
\cline{2-6}
& \(\nu = 1\) &  \(\nu = 5\) & 	\(\nu = 10\) & 	\(\nu = 15\) &	\(\nu = 20\)\\
\hline		
H-S & 95.83&95.83&93.33&89.58&\textbf{100.00} \\
S-T-B & \textbf{98.96}&97.22&97.22&91.67&93.52\\
H-S-T-B	& 87.5&84.72&\textbf{93.75}&87.5&91.67\\	
\hline
\end{tabular}
}
\end{table}

We apply the same set of multivariate point-valued time series classification methods as described in Section \ref{sec53} to classify the representative points of the intervals. The classification accuracies are provided in the following table.
\begin{table}[H]
\scriptsize
\setlength{\abovecaptionskip}{0pt}%
\setlength{\belowcaptionskip}{3pt}%
\centering
\caption{Classification accuracy based on point-value time series methods.} \label{realmupoint}
\setlength{\tabcolsep}{0.2mm}{ 
\begin{tabular}{l|cccc|cccc|cccc|cccc}
\hline
\multirow{2}{*}{Methods} & \multicolumn{4}{c|}{H-S} & \multicolumn{4}{c|}{S-T-B} & \multicolumn{4}{c|}{H-S-T-B} & \multicolumn{4}{c}{Fine-grained}\\
&  \(c\)&\(r\)&\(l\)& \(u\) & \(c\)&\(r\)&\(l\)& \(u\) & \(c\)&\(r\)&\(l\)& \(u\) & \(c\)&\(r\)&\(l\)& \(u\) \\
\hline
KN &92.06&90.48&88.89&92.06&\textbf{100.0}0&66.17&\textbf{100.00}&98.50&96.03&67.55&96.69&93.38&77.78&94.44&77.78&90.74\\
ITDE &76.19&69.84&80.95&82.54&64.66&42.86&59.40&72.18&52.98&46.36&56.29&58.28&51.85&66.67&59.26&50.00\\
Bagging &73.02&79.37&71.43&73.02&31.58&36.84&36.09&33.83&29.8&37.75&36.42&47.68&83.33&83.33&83.33&83.33\\
WE & 93.65&74.60&87.30&90.48&58.65&48.12& 58.65&56.39&61.59&45.70&58.28&54.97&83.33&83.33&83.33&83.33\\
\hline 
FCN & \textbf{100.0}&\textbf{100.0}&\textbf{100.0}&\textbf{100.0}&\textbf{100.0}&\textbf{100.0}&\textbf{100.0}&\textbf{100.0}&\textbf{99.17}&\textbf{99.17}&96.69&\textbf{99.17}&\textbf{97.62}&95.24&92.86&\textbf{97.62} \\
Inception & 98.00&96.00&98.00&96.00&98.11&94.34&97.17&91.50&97.52&\textbf{99.17}&\textbf{99.17}&98.35&92.86&95.24&88.10&95.24\\
XCoord & 94.00&92.00&94.00&96.00&92.45&95.28&98.11&98.11&98.35&97.52&98.35&97.52&90.48&90.48&90.48&92.86\\
MLP & 88.00&94.00&86.00&98.00&97.17&93.40&84.90&93.40&94.21&96.69&95.04&96.69&85.71&\textbf{97.62}&85.71&95.24\\
RNN& 92.00&94.00&92.00&90.00&89.62&93.40&90.57&95.28&86.78&94.22&92.56&89.27&88.10&95.24&88.10&95.24\\
LSTM & 88.00&94.00&92.00&94.00&91.51&99.06&\textbf{100.0}&98.11&95.04&95.04&95.87&95.87&88.10&97.62&88.10&\textbf{97.62}\\
GRU & 86.00&98.00&92.00&96.00&97.17&98.11&98.11&99.06&94.22&95.04&95.04&95.04&88.10&95.24&88.10&95.24\\
RNN\_FCN & 94.00&96.00&94.00&98.00&99.06&98.11&99.06&98.11&96.70&98.35&98.35&99.17&97.62&97.62&92.86&95.24\\
LSTM\_FCN & 96.00&98.00&98.00&98.00&99.06&99.06&99.06&98.11&96.69&\textbf{99.17}&\textbf{99.17}&\textbf{99.17}&95.24&97.62&92.86&92.86\\
GRU\_FCN  &98.00&98.00&98.00&98.00&99.06&98.11&91.51&85.85&91.74&98.35&97.52&98.35&88.10&88.10&90.48&92.86\\
MRNN\_FCN &94.00&96.00&96.00&98.00&96.23&98.11&87.74&99.06&98.35&98.35&98.35&98.35&88.10&97.62&92.86&\textbf{97.62}\\
MLSTM\_FCN &94.00&96.00&96.00&98.00&92.45&98.11&96.22&99.06&98.35&98.35&98.35&97.52&90.48&95.24&\textbf{95.24}&95.24\\
MGRU\_FCN & 92.00&98.00&96.00&98.00&99.06&98.11&98.11&99.06&98.35&\textbf{99.17}&97.52&\textbf{99.17}&95.24&97.62&92.86&\textbf{97.62}\\
ResCNN &96.00&98.00&96.00&98.00&98.11&98.11&94.34&99.06&95.04&98.35&98.35&98.35&95.24&95.24&90.48&92.86\\
ResNet &96.00&96.00&94.00&96.00&98.11&98.11&98.11&99.06&95.87&95.87&96.69&\textbf{99.17}&95.24&95.24&90.48&95.24\\
\hline
\end{tabular}
}
\end{table}

From Table \ref{realJRPIJRP}, it is evident that the classification accuracy based on JRP is generally above 90\%. At \(\nu=30\), the classification accuracy for combination H-S can reach 100\%, and for the most complex combination H-S-T-B, the accuracy can also be quite high at 93.75\%. Looking at Table \ref{realmupoint}, it is observed that deep learning methods and some non-deep learning methods exhibit strong classification capabilities for all three combinations and outperform the proposed method. Among these methods, FCN and its combinations with other neural networks such as LSTM\_FCN, GRU\_FCN, and MGRU\_FCN demonstrate the best classification performance. Additionally, compared to the results in Section \ref{sec61}, the classification ability of the temperature feature is insufficient, and the classification performance could be significantly improved by utilizing all six variables simultaneously.

\subsection{Fine-grained classification}
In addition to the aforementioned classification tasks, we further explore the fine-grained classification of interval-valued time series to thoroughly evaluate the effectiveness of the proposed method. Fine-grained classification involves handling data with high similarity or correlation during the classification process. For this purpose, we collected six weather indicators from the website \url{https://rp5.ru/} for two time periods: from January 2, 2005, to April 11, 2023, for the Shanghai Baoan station, and from July 7, 2019, to April 11, 2023, for the Shanghai Hongqiao station. These data points were then compressed into multivariate interval-value time series at a daily granularity. Similar to the previous approach, we treated 30 days of observations as a trajectory to construct the multivariate interval-valued time series data for each of the two districts.

Given the proximity of both the Baoan and Hongqiao districts within Shanghai, their geographic closeness results in a relatively similar climate. Thus the six weather indicators exhibit a high degree of correlation, which poses a challenge to classification. Moreover, the available data from the website for the Hongqiao district is considerably less compared to Baoan district, introducing a class imbalance that further complicates the classification task. We present the fine-grained discriminant results based on JRP in the following table. The classification results for the representative points are presented in the latter part of Table \ref{realmupoint}.

\begin{table}[H]
\setlength{\abovecaptionskip}{0pt}%
\setlength{\belowcaptionskip}{3pt}%
\centering
\caption{Classification accuracy based on WS-DAN using JRP} \label{realfenJRPIJRP}
\setlength{\tabcolsep}{5mm}{
\begin{tabular}{ccccc}
\hline
\multicolumn{5}{c}{\(\nu\)(JRP) }\\
\hline
\(\nu = 1\) & 	\(\nu = 5\) & 	\(\nu = 10\) & 	\(\nu = 15\) &	\(\nu = 20\)\\
\hline		
91.67 & 83.34 & 87.5& 91.67&  \textbf{95.83}\\
\hline
\end{tabular}
}
\end{table}
Based on the accuracy of fine-grained classification mentioned above, it can be observed that the image-based classification methods yield commendable results. In the case of JRP-based classification, the optimal accuracy of 95.83\% is achieved when \(\nu = 20 \). Furthermore, methods based on deep learning exhibit strong discriminatory capabilities for representative points, particularly in terms of the range (\(r\)) and upper bounds (\(u\)). Similar to the results in Section \ref{sec62}, the FCN method and its combinations with other network structures demonstrate the best discriminatory performance.

\section{Conclusions}\label{sec7}
In this paper, we introduced classification methods for both univariate and multivariate interval-valued time series based on time series imaging. Extensive simulation studies and real-world data applications have demonstrated the effectiveness of the proposed approach in addressing complex discrimination scenarios, and have showcased its significant superiority over representative point-based classification methods. Especially in more complex scenarios, such as when data from different classes exhibit high similarity, both deep learning and non-deep learning methods fail to perform effectively, but the proposed method continues to achieve a satisfactory level of discriminative accuracy, which highlights the robustness of the proposed methods.

In this study, we have employed the imaging methods RP and JRP for classification. However, there are also ongoing research efforts exploring other time series imaging methods, such as Gramian Angular Summation/Difference Field, and Markov Transition Field \cite{wang2015imaging}. While our approach and theory can naturally extend to these alternative imaging methods, it is important to acknowledge that different imaging techniques capture distinct information from the original data. As a result, their accuracy in practical applications may vary, which requires further research and analysis. While we optimize the convex combination coefficients as learnable parameters, {fundamentally we still use points to represent intervals}. This approach does not fully exploit the information within intervals. Future research could consider extending point-valued time series imaging methods to interval scenarios for classification by utilizing distance metrics between intervals. This would enable a more comprehensive utilization of interval information and potentially lead to further enhancements in discriminative performance.

In conclusion, this work represents an initial attempt where we have integrated deep learning techniques to process interval-valued time series. The realm of point-valued time series methods is extensive and rapidly evolving, requiring us to further harness and integrate these methods effectively.

\section*{Acknowledgments}
\addcontentsline{toc}{section}{Acknowledgments}
The research work described in this paper was supported by the National Natural Science Foundation of China (Nos. 72071008). 

\newpage
\begin{appendix}		
\section{Proofs for Results}\label{appendixA}
\setcounter{lemma}{0}
\renewcommand{\thelemma}{\Alph{section}\arabic{lemma}}

Before proving the theoretical results, we first introduce the fact that the three-dimensional convolution operation can be mapped into two matrix multiplications. For the \(l\)-th layer of the CNNs, the convolution operation is
\[
X_{(l)}(:, :, j) = \Theta^j_{(l)}  \circledast Z_{(l-1)} \in \mathbb{R}^{M_l\times N_l}, \ j= 1,2,\cdots, d_l,
\]
where the notations are described in Section \ref{sec4Theoretical}. {Then we} vectorize the \(d_l\) convolution kernels as rows of what is called the \(l\)-th layer weight matrix \(W_{(l)}\), i.e.,
\begin{equation} \label{kernelrearange}
W_{(l)} = \left[
\begin{array}{c}
\vecn(\Theta^1_{(l)})\\
\vecn(\Theta^2_{(l)})\\
\vdots\\
\vecn(\Theta^{d_l}_{(l)})
\end{array}
\right]^\top\in \mathbb{R}^{k_lk_ld_{l-1} \times d_l}.
\end{equation}

Furthermore, we can vectorize the input \(Z_{(l-1)}\) of the \(l\)-th layer as the feature matrix \(F_{(l)}\), that is,
\begin{equation} \label{featurearange}
\begin{aligned}
F_{(l)} &= \left[
\begin{array}{c}
\vecn(Z_{(l)}(1:k	_l,1:k_l,:))\\
\vecn(Z_{(l)}(2:(k_l+1),2:(k_l+1),:))\\
\vdots\\
\vecn(Z_{(l)}((\widetilde{M}_{(l-1)}-k_l+1):\widetilde{M}_{(l-1)}, (\widetilde{N}_{(l-1)}-k_l+1):\widetilde{N}_{(l-1)}, :))  \\
\end{array}
\right]\\
& \in \mathbb{R}^{(\widetilde{M}_{(l-1)}-k_l+1)(\widetilde{N}_{(l-1)}-k_l+1) \times k_l k_l d_{l-1}}.
\end{aligned}
\end{equation}

Multiplying the feature matrix \(F_{(l)}\) and the weight matrix \(W_{(l)}\) {we obtain} a matrix of size \((\widetilde{M}_{(l-1)}-k_l+1)(\widetilde{N}_{(l-1)}-k_l+1)\times d_l\). {Then by reshaping each of its rows into a matrix of size \((\widetilde{M}_{(l-1)}-k_l+1) \times (\widetilde{N}_{(l-1)}-k_l+1)\) and rearranging them, we can get \(X_{(l)}\).}

\begin{lemma} \label{poolingLipschitz}
If \(\sigma: \mathbb{R}^{M\times N\times d}\to \mathbb{R}^{M\times N\times d}\) is a element-wise \(\eta\)-Lipschitz continuous function, then the composite operation of max pooling with window size \(m\times m\) and \(\sigma(\cdot)\) is \(m^{1/p}\eta\)-Lipschitz continuous with respect to \(\lVert \cdot \rVert_p\) for \(p \geq 1\), where \(m\) is the maximum times an entry in the tensor can be reused in a pooling operation.
\end{lemma}

\begin{proof}
For any tensor \(X, Y\in\mathbb{R}^{M\times N\times d} \), {we have}
\[
\begin{aligned}
\lVert \tp(\sigma(X)) - \tp(\sigma(Y))\rVert_p &= \left(\sum_{i,j,k} \left|\max \sigma(X(\mathcal{A}_{i,j,k})) - \max \sigma(Y(\mathcal{A}_{i,j,k}))\right|^p \right)^{1/p}\\
& \leq \left(\sum_{i=1}^{M} \sum_{j=1}^{M} \sum_{k=1}^{d}m \eta \left|X_{i,j,k}- Y_{i,j,k}\right|^p \right)^{1/p}\\
&=m^{1/p}\eta \lVert X - Y\rVert_p,
\end{aligned}
\]
where \(\mathcal{A}_{i,j,k}\) denotes the index set of the sub-tensor used for the pooling operation. Various entities in the tensor may be reused multiple times due to the stride in the pooling operation, and this is the cause of the fourth inequality. 
\end{proof}

In Lemma \ref{poolingLipschitz}, if the pooling operation without overlap, that is, all entities are used up to once, then the composite of pooling operation and \(\sigma(\cdot)\) is also \(\eta\)-Lipschitz.

\noindent
{\bf Lemma 4.1.} {\it 
Given dataset \(I = \{(R_i, Y_i)\}^n_{i=1}\), and any margin \(\gamma > 0\), for any \(\delta \in (0, 1)\), with probability at least \(1-\delta\), every \(f \in \mathbb{F}\) satisfies
\[
\mathbb{P}\left(\argmax_{j}f(R; \Theta)_j \neq Y\right) -  \widehat{\mathcal{R}}(f) \leq  2\Re(\mathcal{H}_\gamma|_I) + \sqrt{\frac{\log (1/\delta)}{2n}}
\]
and
\[
\mathbb{P}\left(\argmax_{j}f(R; \Theta)_j \neq Y\right) -  \widehat{\mathcal{R}}(f) \leq  2  \widehat{\Re}(\mathcal{H}_\gamma|_{I})+ 3\sqrt{\frac{\log (2/\delta)}{2n}}.
\]
}
\begin{proof}
Firstly, {we have}
\begin{equation}\label{Firstlyine}
\begin{aligned}
\mathbb{P}\left(\argmax_{i}f(R)_i \neq Y\right) &\leq \mathbb{P}\left(\max_{i\neq y}f(R)_i \geq f(R)_Y\right)\\
& = \mathbb{E}\mathbb{I}\left\{-\mathcal{M}(f(R, Y)) \geq 0\right\}\\
& \leq \mathbb{E}\ell_\gamma(-\mathcal{M}(f(R), Y))\\
&= \mathcal{R}_\gamma(f).
\end{aligned}
\end{equation}

Suppose the data set \(I^\prime\) differs from the data set \(I\) by only one sample, say \((R^\prime_j, Y^\prime_j)\) in \(I^\prime\) and \((R_j, Y_j)\) in \(I\). Then we can obtain
\[
\begin{aligned}
&\left(\sup_{f \in \mathcal{H}}(\mathcal{R}_\gamma(f) - \widehat{\mathcal{R}}_\gamma(f, I))\right) - \left(\sup_{f \in \mathcal{H}}(\mathcal{R}_\gamma(f) - \widehat{\mathcal{R}}_\gamma(f, I^\prime))\right)\\
& \leq \sup_{f \in \mathcal{H}} \left((\mathcal{R}_\gamma(f) - \widehat{\mathcal{R}}_\gamma(f, I)) - (\mathcal{R}_\gamma(f) - \widehat{\mathcal{R}}_\gamma(f, I^\prime))\right)\\
& \leq \frac{1}{n},
\end{aligned}
\]
where the last inequality holds due to \(|\ell_\gamma(\cdot)| \leq 1\). {Similarly, we can infer that}
\[
\sup_{f \in \mathcal{H}}(\mathcal{R}_\gamma(f) - \widehat{\mathcal{R}}_\gamma(f, I^\prime))- \sup_{f \in \mathcal{H}}(\mathcal{R}_\gamma(f) - \widehat{\mathcal{R}}_\gamma(f, I)) \leq 1/n.
\]

{Thus we can get }
\[
\left|\sup_{f \in \mathcal{H}}(\mathcal{R}_\gamma(f) - \widehat{\mathcal{R}}_\gamma(f, I^\prime))- \sup_{f \in \mathcal{H}}(\mathcal{R}_\gamma(f) - \widehat{\mathcal{R}}_\gamma(f, I))\right| \leq 1/n.
\]

For any \(\delta \in (0,1)\), it follows from McDiarmi's inequality that
\begin{equation} \label{McDiarmi}
\sup_{f \in \mathcal{H}}(\mathcal{R}_\gamma(f) - \widehat{\mathcal{R}}_\gamma(f, I)) \leq \mathbb{E}_{I}\left[\sup_{f \in \mathcal{H}}(\mathcal{R}_\gamma(f) - \widehat{\mathcal{R}}_\gamma(f, I))\right] + \sqrt{\frac{\log(2/\delta)}{2n}}.
\end{equation}
holds with probability at least \(1-\delta/2\)

For the expectation term on the right-hand side of Eq.(\ref{McDiarmi}), {since the samples in both \(I\) and \(I^\prime\) are sampled independently and identically distributed and the supremum function is convex, it
follows from Jensen's inequality that}
%\overset{\textcircled{\small{1}}}{=}
\begin{equation} \label{sup1}
\begin{aligned}
&\mathbb{E}_{I}\left[\sup_{f \in \mathcal{H}}(\mathcal{R}_\gamma(f) - \widehat{\mathcal{R}}_\gamma(f, I))\right]   =       \mathbb{E}_{I}\left[\sup_{f \in \mathcal{H}} \mathbb{E}_{I^\prime}(\widehat{\mathcal{R}}_\gamma(f, I^\prime) - \widehat{\mathcal{R}}_\gamma(f, I))\right]\\
& \leq  \mathbb{E}_{I,I^\prime}\left[\sup_{f \in \mathcal{H}} (\widehat{\mathcal{R}}_\gamma(f, I^\prime) - \widehat{\mathcal{R}}_\gamma(f, I))\right]\\
& = \mathbb{E}_{I,I^\prime}\left[\sup_{f \in \mathcal{H}}
\frac{1}{n} \sum_{i=1}^{n}\left(\ell_\gamma(-\mathcal{M}(f(R^\prime_i), Y^\prime_i)) - \ell_\gamma(-\mathcal{M}(f(R_i), Y_i))\right)\right]\\
& =\mathbb{E}_{u, I,I^\prime}\left[\sup_{f \in \mathcal{H}}
\frac{1}{n} \sum_{i=1}^{n}u_i\left(\ell_\gamma(-\mathcal{M}(f(R^\prime_i), Y^\prime_i)) - \ell_\gamma(-\mathcal{M}(f(R_i), Y_i))\right)\right]\\
& \leq \mathbb{E}_{u, I^\prime}\left[\sup_{f \in \mathcal{H}}
\frac{1}{n} \sum_{i=1}^{n}u_i\ell_\gamma(-\mathcal{M}(f(R^\prime_i), Y^\prime_i))\right] + \mathbb{E}_{u, I}\left[\sup_{f \in \mathcal{H}}
\frac{1}{n} \sum_{i=1}^{n}-u_i\ell_\gamma(-\mathcal{M}(f(R_i), Y_i))\right]\\
&= 2\mathbb{E}_{u, I}\left[\sup_{f \in \mathcal{H}}
\frac{1}{n} \sum_{i=1}^{n}u_i\ell_\gamma(-\mathcal{M}(f(R_i), Y_i))\right]\\
&=2\Re(\mathcal{H}_\gamma|_{I}). 
\end{aligned}
\end{equation}

Combining Eqs.(\ref{Firstlyine}), (\ref{McDiarmi}) and (\ref{sup1}), {we can infer that}
\[
\begin{aligned}
\mathbb{P}\left[\argmax_{i}f(R)_i \neq Y\right] - \widehat{\mathcal{R}}_\gamma(f, I) &\leq \mathcal{R}_\gamma(f) - \widehat{\mathcal{R}}_\gamma(f, I)\\
&\leq \sup_{f \in \mathcal{H}}(\mathcal{R}_\gamma(f) - \widehat{\mathcal{R}}_\gamma(f, I))\\
& \leq 2\Re(\mathcal{H}_\gamma|_{I}) + \sqrt{\frac{\log(1/\delta)}{2n}}
\end{aligned}
\]
holds with probability at least \(1-\delta\). In a similar vein, we have
\[
\widehat{\Re}(\mathcal{H}|_{I}) - \widehat{\Re}(\mathcal{H}|_{I^\prime}) \leq \mathbb{E}_u \left[\sup_{f \in \mathcal{H}}
\frac{1}{n}u_j\left(\ell_\gamma(-\mathcal{M}(f(R_j), Y_j)) - \ell_\gamma(-\mathcal{M}(f(R^\prime_j), Y^\prime_j))\right)\right]\leq \frac{1}{n}.
\]

By using the dds McDiarmi's inequality again, {it is easy to obtain that} \(\Re(\mathcal{H}_\gamma|_{I}) \leq \widehat{\Re}(\mathcal{H}_\gamma|_{I}) + \sqrt{\frac{\log(2/\delta)}{n}}\) holds with probability at least \(1-\delta/2\), combining with Eq.(\ref{McDiarmi}) the conclusion holds.
\end{proof}

\noindent
{\bf Theorem 4.1}
{\it
The matrices \(W_{(1)} \in \mathbb{R}^{nk_1k_1 \times d_1}\) and \(F_{(0)} \in  \mathbb{R}^{(M_0 -k_1 + 1)(N_0 - k_1 + 1) \times nk_1 k_1}\) are rearrangements of \(\Theta_{(1)}\) and \(Z_{(0)}\) based on equations (1) and (2) in Supplementary Material, respectively. Then
\[
\log \mathcal{N}\left(\left\{\Theta_{(1)} \circledast Z_{(0)} , \lVert W_{(1)} \rVert_2 \leq b_1\right\}, \epsilon, \lVert \cdot \rVert_2\right) \leq \bigg\lceil \frac{\lVert F_{(0)} \rVert_2^2b_1^2d_1}{\epsilon^2}\bigg \rceil\ln(2nk_1k_1d_1).
\]
}

\begin{proof}
{Firstly}, \(Z_{(0)} \in \mathbb{R}^{M_0 \times N_0 \times n }\) {is} the tensor composed of input samples \(\{R_i\}^n_{i=1}\). {Rearrange the convolution kernels according to equations (\ref{kernelrearange}) and (\ref{featurearange})} and input \(Z_{(0)}\) of the 1st layer to obtain \(W_{(1)} \in \mathbb{R}^{nk_1k_1 \times d_1}\) and \(F_{(0)} \in  \mathbb{R}^{(M_0 -k_1 + 1)(N_0 - k_1 + 1) \times nk_1 k_1}\), respectively.
%Set \(\omega = \lceil a^2b^2m^{r/2}/\epsilon^2 \rceil\) and \(\bar{a} = am^{1/r}\lVert \mF_{(1)} \rVert\). 
Let \(\widetilde{F}_{(0)}\) be the matrix obtained by rescaling the columns of \(F_{(0)}\), that is, \(\widetilde{F}_{(0)}(:, j) = F_{(0)}(:, j) / \lVert F_{(0)}(:, j) \rVert_2, j = 1,2,\cdots, nk_1k_1\), and define
\[
\begin{aligned}
\{V_1, V_2, \cdots, V_{2n k_1k_1d_1}\} \coloneqq \left\{u\widetilde{F}_{(0)}\ve_i \ve^\top_j:u \in \{-1, 1\},j\in \{1,2,\cdots, d_1\},i\in \{1, 2,\cdots, nk_1k_1\}\right\}.
\end{aligned}
\]

In addition, {the matrix \(B \in \mathbb{R}^{nk_1k_1\times d_1}\) is defined as follows, whose element of the \(j\)-th row is \(\lVert F_{(0)}(:, j) \rVert_2\), i.e.,}
\[
B = \left[
\begin{array}{cccc}
\lVert F_{(0)}(:, 1) \rVert_2 & \lVert F_{(0)}(:, 1) \rVert_2 &\cdots & \lVert F_{(1)}(:, 1) \rVert_2\\
\lVert F_{(0)}(:, 2) \rVert_2 & \lVert F_{(0)}(:, 2) \rVert_2 &\cdots & \lVert F_{(1)}(:, 2) \rVert_2\\
\vdots & \vdots &\ddots & \vdots\\
\lVert F_{(0)}(:, nk_1k_1) \rVert_2 &\lVert F_{(0)}(:, nk_1k_1) \rVert_2 & \cdots & \lVert F_{(0)}(:, nk_1k_1) \rVert_2\\
\end{array}
\right].
\]

Then we have \(F_{(0)} W_{(1)} =\widetilde{F}_{(0)}(B \odot W_{(0)} )\). Using conjugacy of \(\lVert \cdot \rVert_{2, 2}\) and \(\lVert \cdot \rVert_{2, 2}\), {we obtain}
\[
\begin{aligned}
&\lVert B \odot W_{(1)} \rVert_1 \leq \langle B, |W_{(1)}|\rangle \leq \lVert B \rVert_{2, 2} \lVert W_{(1)} \rVert_{2, 2}\\
& = \lVert (\lVert B(:, 1) \rVert_2, \lVert B(:, 2) \rVert_2, \lVert B(:, d_1) \rVert_2) \rVert_2 \lVert W_{(1)} \rVert_{2, 2}\\
& = \lVert (\lVert \lVert \lVert F_{(0)}(:, 1) \rVert_2, \cdots, \lVert F_{(0)}(:, nk_1k_1) \rVert_2  \rVert_2, \lVert B(:, 2) \rVert_2, \lVert B(:, d_1) \rVert_2) \rVert_2 \lVert W_{(1)} \rVert_{2, 2}\\
& = d_1^{1/2} \lVert \lVert F_{(0)}(:, 1) \rVert_2, \cdots, \lVert F_{(0)}(:, nk_1k_1) \rVert_2  \rVert_2\lVert W_{(1)} \rVert_{2, 2}\\
& =d_1^{1/2}\lVert F_{(0)} \rVert_2 \lVert W_{(1)} \rVert_{2, 2}\\
& \leq d_1^{1/2}b_1 \lVert F_{(0)} \rVert_2\coloneqq \bar{a},
\end{aligned}
\]
where \(\lVert B(:, 1) \rVert_2 =\lVert B(:, 2) \rVert_2 = \cdots = \lVert B(:, d_1) \rVert_2\). Furthermore, {we have}
\[
\begin{aligned}
F_{(0)} W_{(1)}&=\widetilde{F}_{(0)}(B \odot W_{(1)}) = \widetilde{F}_{(0)} \sum_{i=1}^{nk_1k_1} \sum_{j=1}^{d_1} (B \odot W_{(1)})(i, j) \ve_i \ve^\top_j\\
& = \lVert B \odot W_{(1)} \rVert_1 \sum_{i=1}^{nk_1k_2} \sum_{j=1}^{d_1} \frac{(B \odot W_{(1)})(i, j)}{\lVert B \odot W_{(1)} \rVert_1} (\widetilde{F}_{(0)}\ve_i \ve^\top_j) \\
& \in \bar{a}\cdot \text{conv}(\{V_1, V_2, \cdots, V_{2nk_1k_1d_1}\}),
\end{aligned}
\]
where \(\text{conv}(\cdot)\) is the convex hull. Using the Lemma 1 in \citet{10.1162/153244302760200713}, there exist nonnegative integers \(\omega_1,\omega_2,\cdots, \omega_{2nk_1k_1d_1}\) with \(\omega = \sum_{i=1}^{2nk_1k_1d_1}\omega_i\), {then we get}
\[
\begin{aligned}
&\bigg \lVert F_{(0)} W_{(1)}  - \frac{\bar{a}}{\omega}\sum_{i=1}^{2nk_1k_1d_1} \omega_iV_i \bigg\rVert^2_2 = \bigg\lVert \widetilde{F}_{(0)}(B \odot W_{(1)}) - \frac{\bar{a}}{\omega}\sum_{i=1}^{2nk_1k_1d_1} \omega_iV_i \bigg\rVert^2_2 \leq \frac{b_1^2 d_1 \lVert F_{(0)} \rVert^2_2}{\omega} \leq \epsilon^2.
\end{aligned}
\]

As \(V_i\) and \(F_{(0)} W_{(1)}\) have the same dimension \((M_0 -k_1 + 1)(N_0 - k_1 + 1) \times d_1\), {by} rearranging these into tensors \(V_i\) of dimension \((M_0 -k_1 + 1) \times (N_0 - k_1 + 1) \times d_1\) , we have 
\[
\begin{aligned}
\bigg\lVert \Theta_{(1)}  \circledast Z_{(0)}- \frac{\bar{a}}{\omega}\sum_{i=1}^{2nk_1k_1d_1}\omega_i V_i \bigg \rVert^2_2 &= \bigg\lVert F_{(0)} W_{(1)}  - \frac{\bar{a}}{\omega}\sum_{i=1}^{2nk_1k_1d_1} \omega_iV_i \bigg\rVert^2_2 \leq \epsilon^2,
\end{aligned}
\]
which implies that the cover element for the convolution operation is thus \(\frac{\bar{a}}{\omega}\sum_{i=1}^{2nk_1k_1d_1}\omega_i V_i\), and
\[
\begin{aligned}
\mathcal{V} \coloneqq \left\{\frac{\bar{a}}{\omega}\sum_{i=1}^{2n k_1k_1d_1}\omega_i V_i:\omega_i\in \mathbb{N}^+, \omega = \sum_{i=1}^{2nk_1k_1d_1}\omega_i\right\} = \left\{\frac{\bar{a}}{\omega}\sum_{j=1}^{\omega} V_{i_j}: (i_i, i_2, \cdots, i_\omega) \in [2nk_1k_1d_1]^\omega\right\}
\end{aligned}
\]
is the desired cover.
\end{proof}

\begin{lemma}
Suppose the non-linear activations of \(i\)th layer are \(\eta_{i}\)-Lipschitz continuous and \(\epsilon_{1}, \epsilon_{2}, \cdots, \epsilon_{L+1}\) denote the cover resolutions of each layer. {Let} 
\[
\epsilon = \eta_{L+1}b_{L+1}\left(\sum_{j\leq L}m^{1/2}_{j}\eta_{j} \epsilon_{j} \prod_{l=j+1}^{i+1} m^{1/2}_{l}\eta_{l} b_{l}\right) + \eta_{L+1} \epsilon_{L+1},
\]
then \(\mathbb{F}\) has the covering number bound
\[
\footnotesize
\begin{aligned}
&\mathcal{N}(\mathbb{F}, \epsilon, \lVert\cdot \rVert_2) \\
&\leq  \sup_{\{\mTheta_{(j)}\}^L_{j=1}, W_{(L+1)}} \mathcal{N}(\left\{W_{(L+1)}[\text{vec}(f^\prime(:, :, :, 1))^\top, \cdots, \text{vec}(f^\prime(:, :, :, n))^\top]: f^\prime\in \mathbb{U}_{L}, \lVert W_{(L+1)} \rVert_\sigma \leq b_{L+1}\right\}, \epsilon_{L+1}, \lVert \cdot \rVert_2)\\
& \times \prod_{i=1}^{L} \sup_{\{\Theta_{(j)}\}^i_{j=1}}\mathcal{N}\left(
\left\{\Theta_{(i)}\circledast \cdots \tp_{(1)}(\sigma_{1}(\Theta_{(1)} \circledast Z_{(0)}) \cdots):\lVert W_{(j)} \rVert_\sigma \leq b_j\right\}, \epsilon_{i}, \lVert \cdot \rVert_2
\right).
\end{aligned}
\]
\end{lemma}

\begin{proof} Firstly, we define
\[
\mathbb{U}_i = \left\{\Theta_{(i)}\circledast \cdots \tp_{(1)}(\sigma_{1}(\Theta_{(1)} \circledast Z_{(0)}) \cdots):\lVert W_{(j)} \rVert_\sigma \leq b_j, j = 1,2,\cdots, i\right\},
\]
{for} \(i = 1,2,\cdots, L\), and define
\[
\mathbb{U} = \left\{ W_{(L+1)}[\text{vec}(f^\prime(:, :, :, 1))^\top, \cdots, \text{vec}(f^\prime(:, :, :, n))^\top]: f^\prime\in \mathbb{U}_{L}, \lVert W_{(L+1)} \rVert_\sigma \leq b_{L+1}\right\}.
\]

Then, we construct the covers \(\mathcal{V}_1, \mathcal{V}_2, \cdots, \mathcal{V}_L, \mathcal{V}\) of \(\mathbb{U}_1, \mathbb{U}_2, \cdots, \mathbb{U}_L, \mathbb{U}\) by induction. For the \(\epsilon_{1}\)-cover \(\mathcal{V}_1\) of \(\mathbb{U}_1\), it is intuitive to have
\[
|\mathcal{V}_1|\leq \mathcal{N}(\{\mathbb{U}_1, \epsilon_{1}, \lVert\cdot \rVert_2\}) \coloneqq N_{\epsilon_{1}}.
\]

For any \(f^\prime \in \mathbb{U}_i\), {we can} construct an \(\epsilon_{i+1}\)-cover \(\mathcal{V}_{i+1}(f^\prime)\) of \(\mathbb{U}_{i+1}\). This corresponds to fixing \(\Theta_{(1)}, \Theta_{(2)}, \cdots, \Theta_{(i)}\) in the construction of \(\epsilon_{i+1}\)-cover \(\mathcal{V}_{i+1}(f)\). {Then we obtain}
\[
|\mathcal{V}_{i+1}(f^\prime)| \leq \sup_{\{\Theta_{(j)}\}^i_{j=1}} \mathcal{N}\left\{\left\{\Theta_{(i+1)}\circledast \tp_{(i)}(\sigma_{i}(f^\prime)):\lVert W_{(i+1)} \rVert_\sigma \leq s_{i+1} \right\}, \epsilon_{i+1}, \lVert \cdot\rVert_2\right\} \coloneqq N_{\epsilon_{i+1}},
\]
and \(\mathbb{U}_{i+1} = \cup_{f^\prime \in \mathbb{U}_i} \mathcal{V}_{i+1}(f^\prime)\), \(|\mathbb{U}_{i+1}| \leq |\mathbb{U}_{i}| \cdot N_{\epsilon_{i+1}} \leq \prod_{j=1}^{i+1}N_{\epsilon_{j+1}}\) {immediately}. As follows, we give the formulation of cover resolutions \(\epsilon\) and show \(\lVert Z_{(L+1)} -\widehat{Z}_{(L+1)} \rVert_2\) with \(\widehat{\mZ}_{(L+1)} \in \mathbb{U}\). When \(i \leq L-1\), for the \(i\)th layer of CNN, we inductively construct \(\widehat{X}_{(i)}\) and \(\widehat{Z}_{(i)}\) {by} 
\[
\widehat{Z}_{(0)} = Z_{(0)},\  \lVert \mTheta_{(i)} \circledast \widehat{Z}_{(i-1)} - \widehat{X}_{(i)}\rVert_2 \leq \epsilon_{i},\ \widehat{Z}_{(i)} = \tp_{(i)}(\sigma_{i}(\widehat{X}_{(i)})),
\]
where \(\widehat{X}_{(i)} \in \mathbb{U}_i\). {Then it follows from Lemma \ref{poolingLipschitz} that}
\[
\begin{aligned}
\lVert Z_{(i+1)} -\widehat{Z}_{(i+1)} \rVert_2 &= \lVert \tp_{(i+1)}(\sigma_{i+1}(X_{(i+1)})) -\tp_{(i+1)}(\sigma_{i+1}(\widehat{X}_{(i+1)})) \rVert_2\\
& \leq m^{1/2}_{i+1}\eta_{i+1} \lVert X_{(i+1)} -\widehat{X}_{(i+1)} \rVert_2\\
& \leq m^{1/2}_{i+1}\eta_{i+1} \lVert X_{(i+1)} - \Theta_{(i+1)} \circledast \widehat{Z}_{(i)} \rVert_2 + m^{1/2}_{i+1}\eta_{i+1} \lVert \Theta_{(i+1)} \circledast \widehat{Z}_{(i)} - \widehat{X}_{(i+1)} \rVert_2\\
& \leq m^{1/2}_{i+1}\eta_{i+1} b_{i+1} \left(\sum_{j\leq i}m^{1/2}_{j}\eta_{j} \epsilon_{j} \prod_{l=j+1}^{i} m^{1/2}_{l}\eta_{l} b_{l}\right) + m^{1/2}_{i+1}\eta_{i+1} \epsilon_{i+1}\\
& = \sum_{j\leq i+1}m^{1/2}_{j}\eta_{j} \epsilon_{j} \prod_{l=j+1}^{i+1} m^{1/2}_{l}\eta_{l} b_{l}.
\end{aligned}
\]

In the case of \(i = L\), {we have}

\[
\begin{aligned}
\lVert Z_{(L+1)} &-\widehat{Z}_{(L+1)} \rVert_2 = \lVert \sigma_{L+1}(X_{(L+1)}) -\sigma_{L+1}(\widehat{X}_{(L+1)}) \rVert_2\\
& \leq \eta_{L+1} \lVert X_{(L+1)} -\widehat{X}_{(i+1)} \rVert_2\\
& \leq \eta_{L+1} b_{L+1} \lVert Z_{(L)}- \widehat{Z}_{(L)}\rVert_2 +\eta_{L+1} \epsilon_{L+1}\\
& \leq \eta_{L+1}b_{L+1}\left(\sum_{j\leq L+1}m^{1/2}_{j}\eta_{j} \epsilon_{j} \prod_{l=j+1}^{i+1} m^{1/2}_{l}\eta_{l} b_{l}\right) + \eta_{L+1} \epsilon_{L+1}.	
\end{aligned}
\]
\end{proof}

{\bf Theorem 4.2}
{
\it Assume that the non-linear activations of the \(i\)th layer is \(\eta_{i}\)-Lipschitz continuous and the number of reuses of each element in the convolution operation of each layer is \(m_i, i = 1,2,\cdots, m_L\) and \(m_{L+1} = 1\). Then the covering number of CNNs class \(\mathbb{F}\) satisfies
\[
\ln \mathcal{N}(\mathbb{F}, \epsilon, \lVert \cdot \rVert_2) \leq \ln\left(2C\widetilde{M}_{L} \widetilde{N}_{L} d_{L} \vee \max_{i \leq L} 2d_ik^2_i\right)  \frac{\lVert Z^k_{(0)} \rVert^2_2 \prod_{j\leq L+1} m_j \eta^2_{j} b^2_j}{\epsilon^2}(L+1)^3.
\]
}

\begin{proof}
{Now we consider the covering} number of the entire CNNs class \(\mathbb{F}\). By utilizing the relationship between the above CNNs classes covering numbers, {we obtain}
\[
\footnotesize
\begin{aligned}
&\ln \mathcal{N}(\mathbb{F}, \epsilon, \lVert\cdot \rVert_2) \\
&\leq  \sup_{\{\Theta_{(j)}\}^L_{j=1}, W_{(L+1)}} \ln \mathcal{N}(\left\{W_{(L+1)}[\text{vec}(f^\prime(:, :, :, 1))^\top, \cdots, \text{vec}(f^\prime(:, :, :, n))^\top]: f^\prime\in \mathbb{U}_{L}, \lVert W_{(L+1)} \rVert_\sigma \leq b_{L+1}\right\}, \epsilon_{L+1}, \lVert \cdot \rVert_2)\\
& + \sum_{i=1}^{L} \sup_{\{\Theta_{(j)}\}^i_{j=1}} \ln\mathcal{N}\left(
\left\{\Theta_{(i)}\circledast \cdots \tp_{(1)}(\sigma_{1}(\Theta_{(1)} \circledast Z_{(0)}) \cdots):\lVert W_{(j)} \rVert_\sigma \leq b_j\right\}, \epsilon_{i}, \lVert \cdot \rVert_2
\right)\\
& =  \sup_{\{\Theta_{(j)}\}^L_{j=1}, W_{(L+1)}} \ln \mathcal{N}(\left\{W_{(L+1)}[\text{vec}(f^\prime(:, :, :, 1))^\top, \cdots, \text{vec}(f^\prime(:, :, :, n))^\top]: f^\prime\in \mathbb{U}_{L}, \lVert W_{(L+1)} \rVert_\sigma \leq b_{L+1}\right\}, \epsilon_{L+1}, \lVert \cdot \rVert_2)\\
&+\sum_{i=1}^{L} \sup_{\{\Theta_{(j)}\}^i_{j=1}} \ln \mathcal{N}\left\{\left\{\text{Rearrange}(W_{(i)}\times F_{(i-1)}):\lVert W_{(i)} \rVert_\sigma \leq b_{i}, f\in \mathbb{U}_{i-1}\right\}, \epsilon_{i}, \lVert \cdot\rVert_2\right\}\\
& \leq \ln\left(2C\widetilde{M}_{L} \widetilde{N}_{L} d_{L} \vee \max_{i \leq L} 2d_ik^2_i\right) \sum_{k=1}^{n}\sum_{i=1}^{L+1} \sup_{\{\Theta_{(j)}\}^L_{j=1}}\frac{b_i^2 \lVert Z^k_{(i-1)} \rVert^2_2 }{\epsilon^2_{i}}\\
& \leq \ln\left(2C\widetilde{M}_{L} \widetilde{N}_{L} d_{L} \vee \max_{i \leq L} 2d_ik^2_i\right) \lVert Z^k_{(0)} \rVert^2_2 \left(\frac{\eta^2_{L+1}b^2_{L+1} \prod_{j\leq L} m_j \eta^2_{j} b^2_j}{\epsilon^2_{L+1}} + \sum_{i=1}^{L}\frac{ \prod_{j\leq i} m_j \eta^2_{j} b^2_j}{\epsilon^2_{i}}\right)\\
&= \ln\left(2C\widetilde{M}_{L} \widetilde{N}_{L} d_{L} \vee \max_{i \leq L} 2d_ik^2_i\right) \lVert Z^k_{(0)} \rVert^2_2 \sum_{i=1}^{L+1}\frac{ \prod_{j\leq i} m_j \eta^2_{j} b^2_j}{\epsilon^2_{i}} \qquad  (\text{set } m_{L+1} = 1).
\end{aligned}
\]

{Specially, we set}
\[
\epsilon_{i} \coloneqq \frac{\zeta_i \epsilon}{\prod_{j>i} m_j \eta^2_{j}b^2_j},\  \zeta_i = \frac{1}{L+1}, \ \zeta = \sum_{i=1}^{L+1}\zeta_i.
\]

{Then we have}
\[
\begin{aligned}
&\ln\left(2C\widetilde{M}_{L} \widetilde{N}_{L} d_{L} \vee \max_{i \leq L} 2d_ik^2_i\right) \lVert Z^k_{(0)} \rVert^2_2 \sum_{i=1}^{L+1}\frac{ \prod_{j\leq i} m_j \eta^2_{j} b^2_j}{\epsilon^2_{i}}\\
&=\ln\left(2C\widetilde{M}_{L} \widetilde{N}_{L} d_{L} \vee \max_{i \leq L} 2d_ik^2_i\right)  \frac{\lVert Z^k_{(0)} \rVert^2_2 \prod_{j\leq L+1} m_j \eta^2_{j} b^2_j}{\epsilon^2}(L+1)^3.
\end{aligned}
\]
\end{proof}

\noindent
{\bf Theorem 4.3}
{\it
Assume that the conditions in Lemma \ref{Rademachercovering} and Theorem \ref{coveringbound} both hold. Then, for any \(\delta \in (0, 1)\), with probability of at least \(1-\delta\), every \(f \in \mathbb{F}\) satisfies
\[
\mathbb{P} \left(\argmax_{i}f(R)_i \neq Y\right)-  \widehat{\mathcal{R}}_\gamma(f) \leq  \frac{24\sqrt{Q}}{n}\left(1 + \ln\left(n/(3\sqrt{Q})\right)\right)+ \sqrt{\frac{\log (1/\delta)}{2n}},
\]
where \(Q = \ln\left(2C\widetilde{M}_{L} \widetilde{N}_{L} d_{L} \vee \max_{i \leq L} 2d_ik^2_i\right)  \frac{4\lVert \mZ^k_{(0)} \rVert^2_2 \prod_{j\leq L+1} m_j \eta^2_{j} b^2_j}{\gamma^2}(L+1)^3\).
}
\begin{proof}
Since function class \(\mathcal{H}_\gamma\) is \(2/\gamma\)-Lipschitz with respect to \(\lVert \cdot \rVert_2\) and the network class \(\mathcal{H}_\gamma\) satisfies the conditions in Theorem 3.2, {we have}
\[
\ln\mathcal{N}\left((\mathcal{H}_\gamma)_I, \epsilon, \lVert \cdot \rVert_2\right) \leq\ln\left(2C\widetilde{M}_{L} \widetilde{N}_{L} d_{L} \vee \max_{i \leq L} 2d_ik^2_i\right)  \frac{4\lVert Z^k_{(0)} \rVert^2_2 \prod_{j\leq L+1} m_j \eta^2_{j} b^2_j}{\gamma^2	\epsilon^2}(L+1)^3.
\]

According to Lemma 3.2, that is, the relationship between the Rademacher complexity and the covering number, {we can get}
\[
\footnotesize
\begin{aligned}
\Re(\mathcal{H}_\gamma|_{I})  &\leq \inf_{\alpha>0}\left(\frac{4\alpha}{n} + \frac{12}{n}\int_{\alpha}^{\sqrt{n}}\sqrt{\log \mathcal{N}(\mathcal{H}_\gamma|_I,\lVert \cdot \rVert_2, \epsilon)}d \epsilon\right)\\
& \leq \inf_{\alpha>0}\left(\frac{4\alpha}{n} + \frac{12}{n}\int_{\alpha}^{\sqrt{n}}\sqrt{\ln\left(2C\widetilde{M}_{L} \widetilde{N}_{L} d_{L} \vee \max_{i \leq L} 2d_ik^2_i\right)  \frac{4\lVert Z^k_{(0)} \rVert^2_2 \prod_{j\leq L+1} m_j \eta^2_{j} b^2_j}{\gamma^2	\epsilon^2}(L+1)^3}d \epsilon\right)\\
& = \inf_{\alpha>0}\left(\frac{4\alpha}{n} + \frac{24}{n} \left(\ln\left(2C\widetilde{M}_{L} \widetilde{N}_{L} d_{L} \vee \max_{i \leq L} 2d_ik^2_i\right)  \frac{\lVert Z^k_{(0)} \rVert^2_2 \prod_{j\leq L+1} m_j \eta^2_{j} b^2_j}{\gamma^2}(L+1)^3\right)^{1/2} \ln(\sqrt{n}/\alpha)\right),
\end{aligned}
\]
{where the} inf is uniquely minimized at 
\[
\alpha = 3 \sqrt{\ln\left(2C\widetilde{M}_{L} \widetilde{N}_{L} d_{L} \vee \max_{i \leq L} 2d_ik^2_i\right)  \frac{4\lVert Z^k_{(0)} \rVert^2_2 \prod_{j\leq L+1} m_j \eta^2_{j} b^2_j}{\gamma^2	n}(L+1)^3}.
\]
\end{proof}

\section{The partial derivatives involved in Section \ref{sec3ADMM}} \label{appendixB}
This section presents the specific expressions of the partial derivatives involved in Section \ref{sec3ADMM}. The specific expression of the partial derivative \(\partial \lVert \vec{C}_{i,j} -\vec{C}_{i, k} \rVert_2 / \partial \alpha_t\) varies for different \(j, k,\) and \(t\). For the sake of simplicity in notation, we omit the index \(i\) referring to the samples as subscripts, {and next we will discuss each case separately}. For \(|j - k| > m\), if \(t = j, j +\kappa, \cdots, j + (m-1)\kappa \), then we have
\[
\frac{\partial \lVert \vec{C}_{j} -\vec{C}_{k} \rVert_2}{\partial \alpha_t} = \frac{(C^l_{j, t} - C^u_{j, t})\left((\alpha_t C^l_{j, t} - (1-\alpha_t)C^u_{j, t}) (\alpha_{k+t-j}C^l_{k, k+t-j} - (1-\alpha_{k+t-j})C^u_{k, k+t-j})\right)}{\lVert \vec{C}_{j} -\vec{C}_{k} \rVert_2}.
\]

If \( t= k, k + \kappa, \cdots, k + (m-1)\kappa\), then we have
\[
\frac{\partial \lVert \vec{C}_{j} -\vec{C}_{k} \rVert_2}{\partial \alpha_t} =\frac{(C^l_{k, t} - C^u_{k, t})\left((\alpha_t C^l_{k, t} - (1-\alpha_t)C^u_{k, t}) (\alpha_{j+t-k}C^l_{k, j+t-k} - (1-\alpha_{j+t-k})C^u_{k, j+t-k})\right)}{\lVert \vec{C}_{j} -\vec{C}_{k} \rVert_2},
\]
and otherwise \(\partial \lVert \vec{C}_{j} -\vec{C}_{k} \rVert_2 /\partial \alpha_t = 0\).   For \(|j-k|\leq m\), if \(t = j\wedge k, j\wedge k + \kappa, \cdots, j \vee k - 1\), then we have
\begin{footnotesize}
	\[
	\frac{\partial \lVert \vec{C}_{j} -\vec{C}_{k} \rVert_2}{\partial \alpha_t}= \frac{(C^l_{j\wedge k, t} - C^u_{j\wedge k, t})\left((\alpha_tC^l_{j\wedge k, t} - (1-\alpha_t)C^u_{j\wedge k, t}) (\alpha_{|k-j|+t}C^l_{j \vee k, |k-j|+t} - (1-\alpha_{|k-j|+t})C^u_{j \vee k, |k-j|+t})\right)}{\lVert \vec{C}_{j} -\vec{C}_{k} \rVert_2}.
	\]
	
\end{footnotesize}
If \(t = j\vee k, j\vee k + \kappa, \cdots, j \wedge k + (m-1)\kappa\), then we have
\begin{footnotesize}
	\[
	\begin{aligned}
		\frac{\partial \lVert \vec{C}_{j} -\vec{C}_{k} \rVert_2}{\partial \alpha_t} &= \frac{(C^l_{j\vee k, t} - C^u_{j\vee k, t})\left((\alpha_tC^l_{j\vee k, t} - (1-\alpha_t)C^u_{j\vee k, t}) (\alpha_{j-k+t}C^l_{j \wedge k, j-k+t} - (1-\alpha_{j-k+t})C^u_{j \wedge k,j-k+t})\right)}{\lVert \vec{C}_{j} -\vec{C}_{k} \rVert_2}\\
		& + \frac{(C^l_{j\wedge k, t} - C^u_{j\wedge k, t})\left((\alpha_tC^l_{j\wedge k, t} - (1-\alpha_t)C^u_{j\wedge k, t}) (\alpha_{k-j+t}C^l_{j \vee k, k-j+t} - (1-\alpha_{k-j+t})C^u_{j \vee k, k-j+t})\right)}{\lVert \vec{C}_{j} -\vec{C}_{k} \rVert_2}.
	\end{aligned}
	\]
	
\end{footnotesize}	
If \(t = j \wedge k + (m-1)\kappa + 1, j \wedge k + m\kappa,  \cdots, j \vee k + (m-1)\kappa\), then we have
\begin{footnotesize}
	\[
	\frac{\partial \lVert \vec{C}_{i,j} -\vec{C}_{i, k} \rVert_2}{\partial \alpha_t} = \frac{(C^l_{j\vee k, t} - C^u_{j\vee k, t})\left((\alpha_tC^l_{j\vee k, t} - (1-\alpha_t)C^u_{j\vee k, t}) (\alpha_{|k-j|+t}C^l_{j \wedge k, |k-j|+t} - (1-\alpha_{|k-j|+t})C^u_{j \wedge k, |k-j|+t})\right)}{\lVert \vec{C}_{j} -\vec{C}_{k} \rVert_2},
	\]
\end{footnotesize}
and otherwise \(\partial\lVert \vec{C}_{j} -\vec{C}_{k} \rVert_2 /\partial \alpha_t = 0\). For multivariate interval-valued time series, considering the approximation of the Heaviside function, the specific expression for the partial derivative \(\partial R_i(k, h) / \partial \alpha_t\) is given by
\[
\begin{aligned}
	&\frac{\partial R_i(k, h)}{\partial \alpha_t} \\
	&= \lim_{\nu \to \infty}\frac{1}{2^p}  \frac{\partial \prod_{j=1}^{p} R_{i,j}(k, h)}{\partial \alpha_t}
	=\lim_{\nu \to \infty}\frac{1}{2^p}  \frac{\partial \prod_{j=1}^{p} (1+ \tanh(\nu (\epsilon_{i,j}- \lVert \vec{D}_{i,j,k} -\vec{D}_{i,j,h} \rVert_2)))}{\partial \alpha_t}\\
	&= \lim_{\nu \to \infty}\frac{\prod_{j\neq t} (1+ \tanh(\nu (\epsilon_{i,j}- \lVert \vec{D}_{i,j,k} -\vec{D}_{i,j,h} \rVert_2)))}{2^N} \frac{\partial (1+ \tanh(\nu (\epsilon_{i,t}- \lVert \vec{D}_{i,t,k} -\vec{D}_{i,t,h} \rVert_2)))}{\partial \alpha_t}\\
	&=-\lim_{\nu \to \infty}\frac{\prod_{j\neq t} (1+ \tanh(\nu (\epsilon_{i,j}- \lVert \vec{D}_{i,j,k} -\vec{D}_{i,j,h} \rVert_2)))}{2^N} \\
	&\times\frac{\nu}{1 + \cosh(2\nu (\epsilon_{i,t}- \lVert \vec{D}_{i,t,k} -\vec{D}_{i,t,h} \rVert_2))} \frac{1}{\lVert \vec{D}_{i,j,k} -\vec{D}_{i,j,h} \rVert_2}\\
	&  \times \sum_{k,h} 2\left(\alpha_t(D^l_{i, t, k} - D^u_{i, t, h}) + (1-\alpha_t)(D^l_{i, t, k} - D^u_{i, t, h})\right) \left((D^l_{i, t, k} - D^u_{i, t, h}) - (D^l_{i, t, k} - D^u_{i, t, h})\right).
\end{aligned}
\]

\end{appendix}

\bibliography{./bib/reference.bib}
\bibliographystyle{elsarticle-num-names}

\end{document}